\documentclass[accepted]{uai2025}

\usepackage[american]{babel}
\usepackage{natbib}
\usepackage{mathtools} 
\usepackage{booktabs}

\usepackage{amsmath,amsfonts}
\usepackage{algorithmic}
\usepackage{algorithm}
\usepackage{array}
\usepackage[caption=false,font=normalsize,labelfont=sf,textfont=sf]{subfig}
\usepackage{textcomp}
\usepackage{stfloats}
\usepackage{xurl}
\usepackage{verbatim}
\usepackage{graphicx}
\usepackage{amsfonts}
\usepackage{amssymb}
\usepackage{amsthm}
\usepackage{soul}
\usepackage{color}
\usepackage{amsmath}
\usepackage{amsthm}
\usepackage{stmaryrd}
\usepackage{bbm}
\usepackage{mathtools}
\usepackage{hyperref}
\usepackage{placeins}
\usepackage{float}
\usepackage{multirow}
\usepackage{wrapfig}

\usepackage[toc,page,header]{appendix}
\usepackage{minitoc}

\SetSymbolFont{stmry}{bold}{U}{stmry}{m}{n}

\theoremstyle{plain}
\newtheorem{theorem}{Theorem}
\newtheorem{proposition}{Proposition}
\newtheorem{lemma}{Lemma}

\theoremstyle{definition}
\newtheorem{definition}{Definition}
\newtheorem{assumption}{Assumption}

\theoremstyle{remark}
\newtheorem{remark}{Remark}

\DeclareMathOperator{\argmax}{arg\,max}
\DeclareMathOperator{\argmin}{arg\,min}
\DeclareMathOperator{\Lip}{Lip}
\DeclareMathOperator{\subjectto}{s.t.}

\title{FDR-SVM: A Federated Distributionally Robust Support Vector Machine via a Mixture of Wasserstein Balls Ambiguity Set}

\author[1]{\href{mailto:mibrahim41@gatech.edu?Subject=Your UAI 2025 FDR-SVM paper}{Michael Ibrahim}{}}
\author[2]{Heraldo Rozas}
\author[1]{Nagi Gebraeel}
\author[1]{Weijun Xie}

\affil[1]{
    H. Milton Stewart School of Industrial and Systems Engineering\\
    Georgia Institute of Technology\\
    Atlanta, Georgia, USA
}
\affil[2]{
    Department of Electrical Engineering\\
    University of Chile\\
    Santiago, Chile
}
  
\begin{document}

\maketitle

\begin{abstract}
We study a federated classification problem over a network of multiple clients and a central server, in which each client’s local data remains private and is subject to uncertainty in both the features and labels. To address these uncertainties, we develop a novel Federated Distributionally Robust Support Vector Machine (FDR-SVM), robustifying the classification boundary against perturbations in local data distributions. Specifically, the data at each client is governed by a unique true distribution that is unknown. To handle this heterogeneity, we develop a novel Mixture of Wasserstein Balls (MoWB) ambiguity set, naturally extending the classical Wasserstein ball to the federated setting. We then establish theoretical guarantees for our proposed MoWB, deriving an out-of-sample performance bound and showing that its design preserves the separability of the FDR-SVM optimization problem. Next, we rigorously derive two algorithms that solve the FDR-SVM problem and analyze their convergence behavior as well as their worst-case time complexity. We evaluate our algorithms on industrial data and various UCI datasets, whereby we demonstrate that they frequently outperform existing state-of-the-art approaches.
\end{abstract}

\addtocontents{toc}{\protect\setcounter{tocdepth}{0}}

\section{Introduction}\label{sec:intro}
Original equipment manufacturers (OEMs) of industrial equipment (used in manufacturing, energy, and healthcare) sell their products to multiple customers under lucrative service contracts that demand stringent reliability standards—an especially critical requirement for systems such as aircraft engines or turbines in power plants. In order to meet these standards, OEMs must provide accurate diagnostics of impending  faults, including their types and severities \citep{dutta2023,yang2025,lei2020}. Most customers are unwilling to share their operational data due to confidentiality and security concerns, particularly in industries critical to national security (e.g., nuclear power plants). Consequently, OEMs face the challenge of leveraging dispersed data sources to develop analytic models capable of improving fault detection and classification.

Distributed fault diagnosis \citep{du2024} via federated learning (FL) \citep{mcmahan2017} offers a promising solution to this problem. FL enables the training of a global model on data that remains at each client, thereby preserving privacy while allowing OEMs to draw on broader information sources for more robust fault diagnosis. Despite these advantages, industrial data can be extremely noisy—due to harsh operating conditions and sensor limitations—and often suffers from labeling errors caused by variations in operator expertise. As a result, industrial datasets tend to be among the \textit{most uncertain and poorly labeled}. 

Uncertainty in a binary classification problem can be modeled by considering the features $\boldsymbol{x} \in \mathcal{X} \subseteq \mathbb{R}^P$ and labels $y \in \{ -1, +1 \}$ as random variables governed by an underlying distribution $\mathbb{P}$ \citep{shafieezadehabadeh2015distributionally}. An ideal classifier with parameters $\boldsymbol{w} \in \mathbb{R}^P$ and a loss function $\ell(\boldsymbol{w};(\boldsymbol{x},y))$ minimizes the expected risk $\mathbb{E}^{\mathbb{P}}[\ell(\boldsymbol{w};(\boldsymbol{x},y))]$. Since $\mathbb{P}$ is typically unknown, it is common practice to rely on an empirical distribution $\widehat{\mathbb{P}}_N$ derived from $N$ IID samples and then minimize the empirical risk $\mathbb{E}^{\widehat{\mathbb{P}}_N}[\ell(\boldsymbol{w};(\boldsymbol{x},y))]$ . However, in cases where the training data is noisy or limited, the resulting model can be highly suboptimal, leading to poor out-of-sample performance \citep{kuhn2019wasserstein,shafieezadehabadeh2015distributionally}.

Distributionally robust optimization (DRO) \citep{scarf1957min,delage2010,bayraksan2015,shapiro2017,datadrivenDRO,2019regularization,kuhn2019wasserstein} addresses these challenges by specifying an ambiguity set $\mathcal{A}$ of plausible data distributions. The model is trained by minimizing the worst-case risk 
$\sup_{\mathbb{Q \in \mathcal{A}}}\mathbb{E}^{\mathbb{Q}}[\ell(\boldsymbol{w};(\boldsymbol{x},y))]$ attained by any distribution $\mathbb{Q} \in \mathcal{A}$. A particularly popular approach is to define $\mathcal{A}$ as a Wasserstein ball around $\widehat{\mathbb{P}}_N$ \citep{kuhn2019wasserstein}. Wasserstein-based DRO (WDRO) has attracted growing attention in machine learning \citep{2019regularization,nieter2023,rui2023}. However, most WDRO research remains limited to centralized settings, and extending it to federated environments introduces significant challenges and computational complexities \citep{cher2020}.

\textbf{Contributions.} We develop a \textit{federated distributionally robust support vector machine} (FDR-SVM) that can be trained to global optimality on data distributed across $G$ clients. Using DRO allows our model to be robust to uncertainties in both features and labels. More importantly, our model does not rely on restrictive assumptions, such as Lipschitz smoothness or strong convexity, which are often imposed by existing FL approaches. Although differential privacy is vital in FL, our work focuses on robustness to distributional uncertainties. To the best of our knowledge, this is the first effort utilizing WDRO to robustify a FL model under such general conditions. The main contributions of this paper are:
\begin{enumerate}
    \item We propose a \textbf{M}ixture \textbf{o}f \textbf{W}asserstein \textbf{B}alls (MoWB) ambiguity set that generalizes the Wasserstein ball to the distributed setting. This lays the foundation for robustifying a variety of FL models under the DRO paradigm. We then prove that the true data distribution belongs to MoWB with a certain confidence level under a mild assumption, and we use it to derive our separable FDR-SVM formulation.

    \item We propose a subgradient method-based (SM) algorithm for training our FDR-SVM, where we rigorously derive the subgradient of the \textit{infinite-dimensional} worst-case risk problem at each client assuming the compactness of the feature support set. We then prove the convergence of this algorithm to global optimality, and derive its worst-case time complexity.

    \item We also propose an alternating direction method of multipliers-based (ADMM) algorithm for training our FDR-SVM, where we derive a convex, tractable optimization problem and a closed-form for local and global model updates, respectively. We show that this algorithm is only guaranteed convergence under the addition of a strongly convex term to each client's objective. While this may affect final model performance, convergence is achieved in fewer rounds than SM and without the need for feature support assumptions.

    \item We evaluate our proposed methods on an industrial dataset and various popular UCI repository datasets, where we study their hyperparameter sensitivity and demonstrate that the FDR-SVM typically outperforms state-of-the-art (SOTA) baselines.
\end{enumerate}
\section{Background and Prior Work}\label{sec:rel_works}
\textbf{Distributionally Robust Optimization.} DRO has gained popularity recently due to its applications in various areas of optimization and ML \citep{kuhn2019wasserstein}. The general 1-WDRO problem is mathematically formulated as
\begin{equation} \label{eq:dro_problem}
    \inf_{\boldsymbol{w} \in \mathcal{W}} \sup_{\mathbb{Q} \in \mathcal{A}_{\varepsilon,1,d}(\Xi)} \mathbb{E}^{\mathbb{Q}}[\ell],
\end{equation}
where $\ell$ is the loss function parameterized by $\boldsymbol{w} \in \mathcal{W}$, and $\mathbb{Q}$ is any distribution within ambiguity set $\mathcal{A}_{\varepsilon,1,d}(\Xi)$, which is defined as
\begin{equation}\label{eq:trad_amb_set}
    \mathcal{A}_{\varepsilon,1,d}(\Xi) \coloneqq \left \{ \mathbb{Q} \in \mathcal{P} \left ( \Xi \right ) \colon W_{d,1} \left ( \mathbb{Q},\widehat{\mathbb{P}}_N\right ) \leq \varepsilon \right \},
\end{equation}
where $\mathcal{P}(\Xi)$ is the set of all distributions supported on $\Xi$, and $W_{d,1}(\cdot,\cdot)$ is the type-$1$ Wasserstein distance equipped with transportation cost function $d(\boldsymbol{\xi},\boldsymbol{\xi}^{\prime})$. It has been shown in various works \citep{datadrivenDRO,kuhn2019wasserstein,shapiro2017,rui2023} that the DRO problem \eqref{eq:dro_problem} admits tractable, convex reformulations in many cases of practical interest. Moreover, it was also demonstrated by \cite{kuhn2019wasserstein} that the Wasserstein ambiguity set enjoys various attractive properties, such as its ability to assign point mass anywhere in the support set, and its interpretation as a confidence interval for $\mathbb{P}$. We provide further background on Wasserstein DRO in Appendix \ref{app:bg}.

Many efforts have successfully utilized DRO to robustify various classifiers. For example, \cite{2019regularization} develop Wasserstein DR logistic regression (LR) and SVM. Further, \cite{selvi2022} extend the DR LR model to data with mixed features. \cite{FACCINI2022} utilize a moment-based ambiguity set to derive a DR version of the SVM. Finally, \cite{sagawa2020} utilize group DRO to mitigate the tendency of classification deep neural networks (DNNs) to learn spurious correlations, relying on manual training data grouping. This is advanced by \cite{wu2023}, who utilize a DNN to perform the data grouping. All these efforts implicitly assume the availability of the training data at a central location, making them difficult to extend to FL settings \citep{cher2020}.

\textbf{Federated Learning.} Since its introduction by \cite{konecny2016e,mcmahan2017}, FL has garnered much attention due to its practical utility. The \texttt{FedAvg} algorithm introduced by \cite{mcmahan2017} relies on local stochastic gradient descent (SGD) updates by clients, and subsequent aggregation and rebroadcasting of model by the server. The work also introduces \texttt{FedSGD}, where each client only performs one local update step. \texttt{FedProx} \citep{li2020} adds a proximal term to the objective function to mitigate client heterogeneity issues. \cite{wang2020b} develop \texttt{FedNova}, where client updates are normalized to address data heterogeneity without impacting convergence. Alternatively, \cite{karimireddy20a} propose \texttt{SCAFFOLD}, where client drift is addressed with the introduction of control variables. Personalized FL algorithms include \texttt{FedPer} \citep{arivazhagan2019} which introduces a local personalization layer at each client, \texttt{FedEM} \citep{marfoq2021} which models local data distributions as a mixture of unknown distributions, \texttt{FedPer++} \citep{xu2022} which utilizes regularization to prevent local overfitting, and \texttt{FedL2P} \citep{lee2023} which uses meta-learning to learn a personalization strategy for each client.

\textbf{Distributionally Robust Federated Learning.} Recently, many efforts have combined ideas from DRO and FL. For example, \cite{deng2020} develop a DR version of \texttt{FedAvg}, hedging against uncertainty in client weights.  \cite{wu2022} propose mixup techniques in the local training stages, addressing noisy and heterogeneous client data. Further, \cite{zecchin2023} develop an efficient algorithm for a DR \texttt{FedAvg} algorithm with no central server. Alternatively, \cite{Huang2021CompositionalFL} combine FL with stochastic compositional optimization (CO), transforming the DR \texttt{FedAvg} algorithm into a CO problem. A \texttt{FedDRO} algorithm is proposed by \cite{Khanduri2023} as an extension of \texttt{FedAvg} for CO problems. \cite{lau2022} construct a Wasserstein ambiguity set from distributed data using barycenters, which may not exist and can be difficult to compute in a distributed fashion if they do. Moreover, \cite{cher2020} and \cite{le2024} propose distributed WDRO formulations. However, the earlier relies on peer-to-peer communication, while the latter assumes the Lipschitz smoothness and strong convexity of the loss function.
\section{Problem Setting} \label{sec:prelim}
We consider the problem of classifying data of the form $\boldsymbol{\xi} = (\boldsymbol{x},y)$ distributed over $G$ clients, where $\boldsymbol{x} \in \mathcal{X} \subseteq \mathbb{R}^P$ is the feature vector and $y \in \{-1,+1\}$ is the label. With such data, a commonly-used transportation cost function is
\begin{equation}\label{eq:cost_func}
d(\boldsymbol{\xi},\boldsymbol{\xi}^{\prime}) \coloneqq ||\boldsymbol{x} - \boldsymbol{x}^{\prime}|| + \kappa \mathbbm{1}_{\{y \neq y^{\prime}\}},
\end{equation}
where $||\cdot||$ is any common norm on $\mathbb{R}^P$, and $\kappa$ is a hyperparameter corresponding to label flipping cost.

We consider classification via the binary SVM characterized by the hinge loss function $\ell_H(\boldsymbol{w};\boldsymbol{\xi})$, which is parameterized by $\boldsymbol{w} \in \mathbb{R}^P$ and defined as $\ell_H(\boldsymbol{w};\boldsymbol{\xi}) = \max\{0,1-y \cdot \boldsymbol{w}^{\mathsf{T}}\boldsymbol{x}\}$. We choose the SVM classifier as it is a well-established model that is commonly used in fault classification settings \citep{dutta2023,josey2018}. Moreover, its simple formulation allows for the rigorous derivation of a DR version.

We study the FL setting where clients can only communicate with the central server but not with each other. Clients do not share their data with the central server, but they can transmit insights from locally trained models, such as local (sub)gradients or model parameters. In this context, we assume the existence of a local training set $\mathcal{S}_g = \{\widehat{\boldsymbol{\xi}}_{n_g}\}_{n_g=1}^{N_g} = \{(\widehat{\boldsymbol{x}}_{n_g},\widehat{y}_{n_g})\}_{n_g=1}^{N_g}$ at each client $g$. We denote the empirical distribution of the $N_g$ IID local training samples and their true distribution as $\widehat{\mathbb{P}}_{N_g}$ and $\mathbb{P}_g$, respectively. Finally, we denote the total number of training samples available at all clients as $N = \sum_g^G N_g$.
\section{MoWB Ambiguity Set}\label{sec:dist_model}

\subsection{Problem Separability} 
In this section, we extend the classical Wasserstein amiguity set to the distributed setting via the novel MoWB ambiguity set $\mathcal{A}_G$ defined next.

\begin{definition}[MoWB ambiguity set]
The Mixture of Wasserstein Balls (MoWB) ambiguity set contains mixture distributions whose constituents are distributions from local Wasserstein balls defined at each client, and is expressed as
\begin{multline} \label{eq:global_amb_set}
    \mathcal{A}_G \coloneqq \bigg \{ \mathbb{Q} \colon \mathbb{Q} = \sum_{g=1}^{G} \alpha_g \mathbb{Q}_g, \ \alpha_g \geq 0, \ \sum_{g=1}^{G} \alpha_g = 1,\\ \ \mathbb{Q}_g \in \mathcal{A}^{(g)}_{\varepsilon_g,1,d}(\Xi)\bigg \},
\end{multline}
where $\alpha_g$ is client $g$'s weight, and $\mathcal{A}^{(g)}_{\varepsilon_g,1,d}(\Xi)$ is the type-$1$ Wasserstein ball of radius $\varepsilon_g$ supported on $\Xi$, centered at $\widehat{\mathbb{P}}_{N_g}$, and defined via cost function $d(\boldsymbol{\xi},\boldsymbol{\xi}^{\prime})$ shown in \eqref{eq:cost_func}.
\end{definition}

\begin{remark} \label{remark:amb_set}
    Observe that when $G=N$, our ambiguity set models worst-case perturbations in individual training samples in a fashion similar to robust optimization (RO). Alternatively, when $G=1$, our ambiguity set reduces to the classical Wasserstein ball $\mathcal{A}_{\varepsilon,1,d}(\Xi)$ defined in \eqref{eq:trad_amb_set}. This suggests that our proposed ambiguity set offers more flexibility in modeling the uncertainty than the classical Wasserstein ambiguity set, which can allow it to achieve improved performance in some settings. This also suggests that our proposed ambiguity set naturally extends the classical Wasserstein ball to the FL setting. Indeed, we show in Proposition \ref{lemma:general_form} that when equipped with the MoWB ambiguity set, the DRO problem enjoys a naturally distributed formulation.
\end{remark}
\begin{proposition}[Problem Separability] \label{lemma:general_form}
    The original DRO problem in \eqref{eq:dro_problem} equipped with the MoWB ambiguity set defined in \eqref{eq:global_amb_set} admits the following reformulation:
    \begin{multline} \label{eq:general_form_eq}
        \inf_{\boldsymbol{w}} \sup_{\mathbb{Q} \in \mathcal{A}_G} \mathbb{E}^{\mathbb{Q}}[\ell_H(\boldsymbol{w};\boldsymbol{\xi})] \\= \inf_{\boldsymbol{w}} \sum_{g=1}^{G} \alpha_g \sup_{\mathbb{Q}_g \in \mathcal{A}^{(g)}_{\varepsilon_g,1,d}(\Xi)}\mathbb{E}^{\mathbb{Q}_g}[\ell_H(\boldsymbol{w};\boldsymbol{\xi})].
    \end{multline}
\end{proposition}
\begin{proof}
    Proof is provided in Appendix \ref{proof:general_form}.
\end{proof}

\subsection{Out-of-Sample Performance Guarantees} 
Since the MoWB ambiguity set $\mathcal{A}_G$ relies on local Wasserstein balls $\mathcal{A}^{(g)}_{\varepsilon_g,1,d}(\Xi)$, it inherits desirable out-of-sample performance guarantees shown by \cite{kuhn2019wasserstein}. Indeed, we show in Proposition \ref{prop:oos} that the true distribution $\mathbb{P} = \sum_{g=1}^G \alpha_g \mathbb{P}_g$ is contained within the MoWB ambiguity set with a certain confidence level, thereby allowing for the reduction of the \textit{true risk} without knowing $\mathbb{P}$. This relies on Assumption \ref{assump:light_tail}, which allows for tighter concentration inequalities for $\mathbb{P}_g$, ensuring that they can indeed be modeled as a perturbation of the empirical distributions $\widehat{\mathbb{P}}_{N_g}$.

\begin{assumption}[Light-tailed Distribution]\label{assump:light_tail}
    The true distribution $\mathbb{P}_g$ of the data at client $g$ is light-tailed. That is, there exists $a > 1$ with $A_g \coloneqq \mathbb{E}^{\mathbb{P}_g}[\exp(||2\boldsymbol{x}||^{a_g})] < +\infty$.
\end{assumption}

\begin{proposition}[Out-of-Sample Performance]\label{prop:oos}
    Suppose Assumption \ref{assump:light_tail} holds and the local Wasserstein ball radius $\varepsilon_g$ at client $g$ is set as \citep{kuhn2019wasserstein}
    \begin{multline*}
    \allowdisplaybreaks
        \varepsilon_{N_g}(\eta_g) = \left ( \frac{\log(c_{1_g} \eta_g^{-1})}{c_{2_g}N_g} \right )^{\frac{1}{a_g}} \mathbbm{1}_{ \left \{N_g < \frac{\log(c_{1_g} \eta_g^{-1)}}{c_{2_g} c_{3_g}} \right \}} \\+ \left ( \frac{\log(c_{1_g} \eta_g^{-1})}{c_{2_g}N_g} \right )^{\frac{1}{P}} \mathbbm{1}_{ \left \{N_g \geq \frac{\log(c_{1_g} \eta_g^{-1)}}{c_{2_g} c_{3_g}} \right \}},
    \end{multline*}
    where $c_{1_g}, c_{2_g}, c_{3_g} \in \mathbb{R}_+$ are constants that depend on $a_g$, $A_g$, $P$ (dimension of the feature space), and the transportation cost given by \eqref{eq:cost_func}. Then the MoWB ambiguity set $\mathcal{A}_G$ defined in \eqref{eq:global_amb_set} enjoys the following property
    \begin{equation*}
    \mathbb{P}^N \{ \mathbb{P} \in \mathcal{A}_G \} \geq \prod_{g=1}^G (1 - \eta_g),
    \end{equation*}
where $\eta_g$ is such that $\mathbb{P}^{N_g} \{ \mathbb{P}_g \in \mathcal{A}^{(g)}_{\varepsilon_g,1,d}(\Xi) \} \geq (1 - \eta_g)$.
\end{proposition}
\begin{proof}
    Proof is provided in Appendix \ref{proof:oos}.
\end{proof}

\section{Solution Algorithms}\label{sec:algs}
We introduce two algorithms to solve problem \eqref{eq:general_form_eq}. Given that our motivating scenario involves manufacturing plants (clients) with ample local compute resources and reliable communication with a central server, we adopt Assumption \ref{assump:sync_train} to guarantee convergence of our algorithms.
\begin{assumption} [Synchronous Training] \label{assump:sync_train}
    The distributed optimization problem in \eqref{eq:general_form_eq} is solved synchronously. That is, the central server only performs an update step once all the clients have completed solving their local problems and communicated their insights to the central server.
\end{assumption}
\subsection{Subgradient-based Algorithm (SM)}\label{subsec:sm_alg}
The subgradient-based (\textbf{SM}) algorithm begins by initializing the global model parameters $\boldsymbol{w}$. Next, each client $g$ seeks to obtain a subgradient for their inner maximization problem from \eqref{eq:general_form_eq}, to be sent to the server for aggregation and model update. This requires each client $g$ to obtain a worst-case distribution $\mathbb{Q}^{\ast}_g$ from its local ambiguity set $\mathcal{A}^{(g)}_{\varepsilon_g,1,d}(\Xi)$, allowing the worst-case risk to be directly expressed as an expectation with respect to $\mathbb{Q}^{\ast}_g$. However, \cite{kuhn2019wasserstein} show that the worst-case distribution cannot be obtained from a type-1 Wasserstein ambiguity set centered around an empirical distribution $\widehat{\mathbb{P}}_{N_g}$ if $\mathcal{X}$ is not compact. Therefore, we make Assumption \ref{assump:feat_supp} \textit{only} for the SM algorithm, ensuring the compactness of $\mathcal{X}$.

\begin{assumption}[Support of Feature Vector] \label{assump:feat_supp}
    The feature vector $\boldsymbol{x}$ is such that: 
        $0 \leq e_p^{\mathsf{T}}\boldsymbol{x} \leq 1 \ \forall p \in [P]$, 
    where $e_p$ are the standard unit vectors.
\end{assumption}

Note that Assumption \ref{assump:feat_supp} is not very restrictive in practice. Indeed, real-world data is often bounded by sensor ranges, and can therefore be easily normalized. Given Assumption \ref{assump:feat_supp} holds and the global model parameters $\boldsymbol{w}$ are fixed, then by \cite{2019regularization} it can be shown that the worst-case distribution $\mathbb{Q}^{\ast}_g$ for client $g$ is
\begin{multline} \label{eq:ext_dist}
        \mathbb{Q}_g^{\ast} = \frac{1}{N_g} \sum_{n_g=1}^{N_g} \bigg ( {\beta_{n_g}^+}^{\ast} \delta_{(\widehat{\boldsymbol{x}}_{n_g}-{\boldsymbol{q}_{n_g}^+}^{\ast}/{\beta_{n_g}^+}^{\ast},\widehat{y}_{n_g})} \\+  {\beta_{n_g}^-}^{\ast} \delta_{(\widehat{\boldsymbol{x}}_{n_g}-{\boldsymbol{q}_{n_g}^-}^{\ast}/{\beta_{n_g}^-}^{\ast},-\widehat{y}_{n_g})} \bigg ),
    \end{multline}
where $\delta_{(\boldsymbol{x},y)}$ is the Dirac density function that assigns probability mass 1 at sample $\boldsymbol{\xi}=(\boldsymbol{x},y)$, and $\beta_{n_g}^+$, $\beta_{n_g}^-$, $\boldsymbol{q}_{n_g}^+$, and $\boldsymbol{q}_{n_g}^-$ are maximizers of the following optimization problem:
\begin{align}\label{eq:subgrad_client_prob}
    \nonumber &\max_{\substack{\beta_{n_g}^+, \beta_{n_g}^- \\ \boldsymbol{q}_{n_g}^+, \boldsymbol{q}_{n_g}^-}} H_g(\boldsymbol{w}) \coloneq\\ &\left \{ \begin{aligned}
        &\max_{\substack{\beta_{n_g}^+, \beta_{n_g}^- \\ \boldsymbol{q}_{n_g}^+, \boldsymbol{q}_{n_g}^-}} && \begin{multlined} \frac{1}{N_g}\sum_{n_g=1}^{N_g} \big ( -(\beta_{n_g}^+ - \beta_{n_g}^-) \widehat{y}_{n_g} \boldsymbol{w}^{\mathsf{T}}\widehat{\boldsymbol{x}}_{n_g} \\- \widehat{y}_{n_g} \boldsymbol{w}^{\mathsf{T}} (\boldsymbol{q}_{n_g}^+ - \boldsymbol{q}_{n_g}^-) \big )\end{multlined}\\
        & \subjectto && \sum_{n_g=1}^{N_g} \big ( ||\boldsymbol{q}_{n_g}^+|| + ||\boldsymbol{q}_{n_g}^-|| + \kappa_g \beta_{n_g}^- \big ) \leq N_g \varepsilon_g \\
        &&& \beta_{n_g}^+ + \beta_{n_g}^- = 1 \qquad \qquad \qquad \  \forall n_g \in [N_g]\\
        &&& 0 \leq \beta_{n_g}^+ \widehat{\boldsymbol{x}}_{n_g}-\boldsymbol{q}_{n_g}^+ \leq \beta_{n_g}^+ \qquad \forall n_g \in [N_g]\\
        &&&0 \leq \beta_{n_g}^- \widehat{\boldsymbol{x}}_{n_g}-\boldsymbol{q}_{n_g}^- \leq \beta_{n_g}^- \qquad \forall n_g \in [N_g]\\
        &&& \beta_{n_g}^+,\beta_{n_g}^- \geq 0 \qquad \qquad \qquad \quad \ \forall n_g \in [N_g]
    \end{aligned} \right.,
\end{align}
where $||\cdot||$ is the norm used in the definition of the transportation cost function \eqref{eq:cost_func}. Armed with the discrete worst-case distribution $\mathbb{Q}^{\ast}_g$, each client $g$ can compute a subgradient, $\boldsymbol{v}_g$, of their local maximization problem. Proposition \ref{prop:subgrad_comp} presents a closed-form for obtaining $\boldsymbol{v}_g$.

\begin{proposition}[Local Subgradient Computation]\label{prop:subgrad_comp}
    Suppose the worst-case distribution $\mathbb{Q}^{\ast}_g$ is known to client $g$. Then, they can compute a subgradient $\boldsymbol{v}_g$ for their respective maximization problem from \eqref{eq:general_form_eq} as any vector that obeys
    \begin{equation*}
        \boldsymbol{v}_g \in \frac{1}{N_g} \sum_{n_g=1}^{N_g} \big ( \mathcal{B}^+ + \mathcal{B}^- \big ),
    \end{equation*}
    where $+$ is the Minkowski sum and $\mathcal{B}^+$, $\mathcal{B}^-$ are defined as
    \begin{equation*}
        \mathcal{B}^{\pm} \coloneq \left \{ \begin{aligned}
            & \boldsymbol{0} && \text{if } \widehat{r}_{n_g}^{\pm} < 0\\
            & \mp {\beta_{n_g}^{\pm}}^{\ast} \widehat{y}_{n_g}\widehat{\boldsymbol{z}}_{n_g}^{\pm}&& \text{if } \widehat{r}_{n_g}^{\pm} > 0\\
            & \text{conv}\left ( \{ \boldsymbol{0}, \mp {\beta_{n_g}^{\pm}}^{\ast} \widehat{y}_{n_g}\widehat{\boldsymbol{z}}_{n_g}^{\pm} \} \right ) && \text{if } \widehat{r}_{n_g}^{\pm} = 0
        \end{aligned} \right.,
    \end{equation*}
where $\widehat{\boldsymbol{z}}_{n_g}^{\pm} = \widehat{\boldsymbol{x}}_{n_g}-{\boldsymbol{q}_{n_g}^{\pm}}^{\ast}/{\beta_{n_g}^{\pm}}^{\ast}$, $\widehat{r}_{n_g}^{\pm} =  1 \mp \widehat{y}_{n_g} \cdot \boldsymbol{w}^{\mathsf{T}} \widehat{\boldsymbol{z}}_{n_g}^{\pm}$, and $\text{conv}(\Theta)$ is the convex hull of set $\Theta$.
\end{proposition}
\begin{proof}
    Proof is provided in Appendix \ref{proof:subgrad_comp}.
\end{proof}

The subgradients $\boldsymbol{v}_g$ from the clients are then aggregated by the server, and used to update the global model $\boldsymbol{w}$ and broadcast it back to the clients. This process repeats for $T$ rounds. The pseudocode of the SM algorithm is given in \ref{alg:subgrad}.
\begin{algorithm}[ht]
    \caption{SM Algorithm}
    \smaller
    \label{alg:subgrad}
    \textbf{Input:} $\boldsymbol{w}^{(0)}$\\
    \textbf{Parameters:} Number of rounds $T$, step-size $\gamma(t)$ at round $t$\\
    \textbf{Output:} $\boldsymbol{w}^{\ast}$
    \begin{algorithmic}[1]
        \FOR{$t = 1, \dots, T$}
        \STATE \textbf{Client Update:}
        \FOR{clients $g = 1, \dots, G$}
        \STATE Solve for $[{\beta_{n_g}^+}^{\ast}, {\beta_{n_g}^-}^{\ast}, {\boldsymbol{q}_{n_g}^+}^{\ast}, {\boldsymbol{q}_{n_g}^-}^{\ast}] \leftarrow \argmax H_g(\boldsymbol{w}^{(t)})$
        \STATE Compute $\mathbb{Q}_g^{\ast} \leftarrow \frac{1}{N_g} \sum_{n_g=1}^{N_g} {\beta_{n_g}^+}^{\ast} \delta_{(\widehat{\boldsymbol{x}}_{n_g}-{\boldsymbol{q}_{n_g}^+}^{\ast}/{\beta_{n_g}^+}^{\ast},\widehat{y}_{n_g})} +  {\beta_{n_g}^-}^{\ast} \delta_{(\widehat{\boldsymbol{x}}_{n_g}-{\boldsymbol{q}_{n_g}^-}^{\ast}/{\beta_{n_g}^-}^{\ast},-\widehat{y}_{n_g})}$
        \STATE Compute any local subgradient $\boldsymbol{v}_g$ via Proposition \ref{prop:subgrad_comp}
        \STATE Send $\boldsymbol{v}_g$ to central server.
        \ENDFOR
        \STATE \textbf{Server Update:}
        \STATE $\boldsymbol{w}^{(t+1)} \longleftarrow \boldsymbol{w}^{(t)} - \gamma(t) \sum_{g=1}^{G} \alpha_g \boldsymbol{v}_g$
        \STATE Broadcast $\boldsymbol{w}^{(t+1)}$ to all clients
        \ENDFOR
    \end{algorithmic}
\end{algorithm}

\textbf{Convergence.} It is known that the subgradient method converges to an optimal objective value under certain conditions \citep{boyd2003}. We present Lemmas \ref{lem:sub_con}, \ref{lem:lip_cont}, \ref{lem:coer} in Appendix \ref{sec:lemmas}, proving the convexity, Lipschitz continuity, and coercivity of problem \eqref{eq:general_form_eq}'s objective in $\boldsymbol{w}$. Theorem \ref{thm:sm_conv} then asserts that these properties satisfy the convergence criteria of the subgradient method given that the step-size diminishes appropriately, proving the convergence of the SM algorithm. We also derive the SM algorithm's worst-case time complexity in Theorem \ref{thm:subgrad_conv}, showing that it converges in polynomial time with a sublinear number of communication rounds.

\begin{theorem}[SM Convergence]\label{thm:sm_conv}
    The SM Algorithm \ref{alg:subgrad} converges to an optimal solution $\boldsymbol{w}^{\ast}$ of problem \eqref{eq:general_form_eq} within an arbitrary tolerance $\epsilon_1 > 0$, provided the step-size $\gamma(t) \rightarrow 0$ as $t \rightarrow \infty$ and $\sum_{t=1}^{\infty} \gamma(t) = \infty$.
\end{theorem}
\begin{proof}
    Proof is provided in Appendix \ref{proof:sm_conv}.
\end{proof}

\begin{theorem}[SM Time Complexity]\label{thm:subgrad_conv}
Suppose the $\ell_{\infty}$-norm is used in problem \eqref{eq:subgrad_client_prob}, and that it is solved via the barrier method equipped with the log barrier function and Newton updates. Then, the SM algorithm \ref{alg:subgrad} with the diminishing step-size in Theorem \ref{thm:sm_conv} has an overall worst-case time complexity of $\mathcal{O}\left (\epsilon_1^{-2} \left [ N_{g^{\ast}}^{3.5}P^{3.5}\log(\epsilon_2^{-1}) + GP \right ] \right )$ (with $\mathcal{O}(\epsilon_1^{-2})$ communication rounds), where $\epsilon_1$, $\epsilon_2 > 0$ are tolerances 
    on the solutions of the subgradient method and problem \eqref{eq:subgrad_client_prob}, respectively, and $N_{g^{\ast}}$ is the greatest number of samples at any client. 
\end{theorem}

\begin{proof}
    Proof is provided in Appendix \ref{proof:subgrad_conv}.
\end{proof}

\subsection{ADMM-based Algorithm (ADMM)}\label{subsec:admm_alg}
The ADMM-based (\textbf{ADMM}) algorithm requires each client $g$ to solve their local problem and send their optimal local model $\boldsymbol{w}^{\ast}_g$ to the server. There, the local models are aggregated to obtain optimal global model $\boldsymbol{w}^{\ast}$, which is broadcast to the clients. This repeats for $T$ rounds. To guarantee theoretical convergence, we also create a modified version of this algorithm with strongly convex client objectives, denoted as \textbf{ADMM-SC}. Further detail on this is given in the convergence discussion. Deriving this algorithm begins by introducing a decision variable $\boldsymbol{w}_g$ for each client $g$, and rewriting problem \eqref{eq:general_form_eq} to enforce client concensus as follows.
\begin{equation}\label{eq:admm_orig_obj}
\begin{aligned}
&\inf_{\boldsymbol{w_g},\boldsymbol{w}} &&\sum_{g=1}^{G} \alpha_g \sup_{\mathbb{Q}_g \in \mathcal{A}^{(g)}_{\varepsilon_g,1,d}(\Xi)}\mathbb{E}^{\mathbb{Q}_g}[\ell_H(\boldsymbol{w}_g;\boldsymbol{\xi})]\\
    &\subjectto && \boldsymbol{w}_g - \boldsymbol{w} = 0 \qquad \qquad \qquad \forall g \in [G].
\end{aligned}
\end{equation}
Next, we express the Augmented Lagrangian parameterized by scale parameter $\rho$ for the problem in \eqref{eq:admm_orig_obj} as follows:
\begin{multline*}
    \allowdisplaybreaks
    \mathcal{L}_{\rho}(\boldsymbol{w}_1, \dots, \boldsymbol{w}_G, \boldsymbol{w}, \boldsymbol{\mu}_1, \dots, \boldsymbol{\mu}_G) \\= \sum_{g=1}^G \alpha_g\mathcal{L}_{\rho_g}(\boldsymbol{w}_g, \boldsymbol{w}, \boldsymbol{\mu}_g),
\end{multline*}
where $\boldsymbol{\mu_g}$ are client $g$'s scaled Lagrange multipliers, and
\begin{multline*}
    \mathcal{L}_{\rho_g}(\boldsymbol{w}_g, \boldsymbol{w}, \boldsymbol{\mu}_g) = \\\sup_{\mathbb{Q}_g \in \mathcal{A}^{(g)}_{\varepsilon_g,1,d}(\Xi)}\mathbb{E}^{\mathbb{Q}_g}[\ell_H(\boldsymbol{w}_g;\boldsymbol{\xi})]
    + \frac{\rho}{2} ||\boldsymbol{w}_g - \boldsymbol{w} + \boldsymbol{\mu}_g||_2^2.
\end{multline*}
Given the Augmented Lagrangian, client $g$ and the server can obtain their model updates by minimizing it with respect to $\boldsymbol{w}_g$ and $\boldsymbol{w}$, respectively. Proposition \ref{prop:admm_client_up} presents a \textit{tractable, convex} problem that is solved by each client for local model updates. Proposition \ref{prop:admm_server_up} presents a closed-form expression for the server's update of the global model.

\begin{proposition}[ADMM Client Update]\label{prop:admm_client_up}
    Provided with the updated global model $\boldsymbol{w}^{\ast}$, client $g$ can obtain their updated local model $\boldsymbol{w}_g^{\ast}$ as the minimizer to the following problem
    \begin{align}
        &  \nonumber J_g (\boldsymbol{w},\boldsymbol{\mu}_g) =  \\
        &\label{eq:admm_2d}  \left \{ \begin{aligned}
            & \min_{\boldsymbol{w}_g,\lambda_g,s_{n_g}}&& \lambda_g \varepsilon_g + \frac{1}{N_g} \sum_{n_g=1}^{N_g}s_{n_g} \\
            &&& \qquad \qquad \qquad  +\frac{\rho}{2} || \boldsymbol{w}_g - \boldsymbol{w}^{\ast} + \boldsymbol{\mu}_g||_2^2\\
            & \subjectto && \ell_H(\boldsymbol{w}_g;(\widehat{\boldsymbol{x}}_{n_g},\widehat{y}_{n_g})) \leq s_{n_g} \ \forall n_g \in [N_g] \\
            &&& \begin{aligned} &\ell_H(\boldsymbol{w}_g;(\widehat{\boldsymbol{x}}_{n_g},-\widehat{y}_{n_g})) - \kappa \lambda_g \leq s_{n_g}\\
            & \qquad \qquad \qquad \qquad \qquad \quad \forall n_g \in [N_g] \end{aligned} \\
            &&& \lambda \geq ||\boldsymbol{w}_g||_{\ast}&&
        \end{aligned} \right.,
    \end{align}
where $||\cdot||_{\ast}$ is the dual to the norm used in the transportation cost function \eqref{eq:cost_func}.
\end{proposition}

\begin{proof}
    Proof is provided in Appendix \ref{proof:admm_client_up}.
\end{proof}

\begin{proposition}[ADMM Server Update]\label{prop:admm_server_up}
    Provided with the updated local models $\boldsymbol{w}_g^{\ast}$ and scaled Lagrange multipliers $\boldsymbol{\mu}_g^{\ast}$, the central server can obtain the updated global model $\boldsymbol{w}^{\ast}$ via the following
    \begin{equation*}
        \boldsymbol{w}^{\ast} = \sum_{g=1}^G \alpha_g (\boldsymbol{w}_g^{\ast} + \boldsymbol{\mu}_g^{\ast}).
    \end{equation*}
\end{proposition}

\begin{proof}
    Proof is provided in Appendix \ref{proof:admm_server_up}.
\end{proof}

The server then broadcasts $\boldsymbol{w}^{\ast}$ to the clients, where they update their Lagrange multipliers and the process repeats. We provide the pseudocode for the ADMM algorithm in \ref{alg:admm}.
\begin{algorithm}[ht]
    \caption{ADMM/ADMM-SC Algorithm}
    \smaller
    \label{alg:admm}
    \textbf{Input:} $\boldsymbol{w}^{(0)}$, $\boldsymbol{w}_g^{(0)}$, $\boldsymbol{\mu}_g^{(0)}$\\
    \textbf{Parameters:} Number of rounds $T$, scale parameter $\rho$\\
    \textbf{Output:} $\boldsymbol{w}^{\ast}$
    \begin{algorithmic}[1]
        \FOR{$t = 1, \dots, T$}
        \STATE \textbf{Client Update:}
        \FOR{clients $g = 1, \dots, G$}
        \STATE Solve for $\boldsymbol{w}_g^{(t+1)} \leftarrow \boldsymbol{w}_g^{\ast} \text{ minimizer of } J_g(\boldsymbol{w}^{(t)}, \boldsymbol{\mu}_g^{(t)})$ \eqref{eq:admm_2d} (or  $J_g^{\text{SC}}(\boldsymbol{w}^{(t)}, \boldsymbol{\mu}_g^{(t)})$ in Appendix \ref{sec:admm_sc} for ADMM-SC) 
        \STATE Send $\boldsymbol{w}_g^{(t+1)}$ to central server
        \ENDFOR
        \STATE \textbf{Server Update:}
        \STATE Update $\boldsymbol{w}^{(t+1)} \leftarrow \sum_{g=1}^G \alpha_g (\boldsymbol{w}_g^{(t+1)} + \boldsymbol{\mu}_g^{(t)})$
        \STATE Broadcast $\boldsymbol{w}^{(t+1)}$ to all clients
        \STATE \textbf{Client Update:}
        \FOR{clients $g = 1, \dots, G$}
        \STATE  $\boldsymbol{\mu}_g^{(t+1)} \leftarrow \boldsymbol{\mu}_g^{(t)} + \boldsymbol{w}_g^{(t+1)} - \boldsymbol{w}^{(t+1)}$
        \ENDFOR
        \ENDFOR
    \end{algorithmic}
\end{algorithm}

\textbf{Convergence.} Even if the objective function is closed and proper convex as we show in Lemma \ref{lem:clo_pro_con} in Appendix \ref{sec:lemmas}, it has been established in the literature that the global convergence of the multi-block (i.e. $G \geq 3$) ADMM algorithm is generally not guaranteed. Indeed, a counterexample is presented by \cite{chen2016} demonstrating that the multi-block ADMM with a separable convex objective function can fail to converge. However, the convergence of multi-block ADMM in practical cases such as \citep{tao2011} has motivated works to investigate conditions under which it is guaranteed convergence \citep{deng2017,lin2015}. In Theorem \ref{thm:admm_conv_cri}, we introduce an additional strongly convex term to be added to each client's objective, and we denote the ADMM algorithm with strongly convex client objectives as ADMM-SC. We then show that ADMM-SC indeed converges as it obeys the criteria given by \cite{lin2015}. Subsequently, we present worst-case time complexity of ADMM-SC in Theorem \ref{thm:admm_conv}, showing that it too converges in polynomial time, but requires fewer communication rounds than the SM algorithm.
\begin{theorem}[ADMM-SC Convergence]\label{thm:admm_conv_cri} Suppose the local client problem in \eqref{eq:admm_2d} is modified by adding a $\tau_g||\boldsymbol{w}_g||_2^2$ term to the objective function, where $\tau_g$ is a user-defined hyperparameter, resulting in the modified client problem with a strongly convex objective $J_g^{\text{SC}}(\boldsymbol{w}^{(t)}, \boldsymbol{\mu}_g^{(t)})$ in Appendix \ref{sec:admm_sc}. Suppose further that the ADMM-SC algorithm in \ref{alg:admm} is used to train the FDR-SVM with the modified objective. Then, ADMM-SC converges to an optimal solution $\boldsymbol{w}^{\ast}$ of the modified problem with arbitrary tolerance $\epsilon_1 > 0$ if $\rho \leq \min_{g=1,\dots,G-1}\Bigg \{ \frac{4\alpha_g \tau_g}{g(2G+1-g)},\frac{4\alpha_G \tau_G}{(G-1)(G+2)}\Bigg \}$.
\end{theorem}

\begin{proof}
    Proof is provided in Appendix \ref{proof:admm_conv_cri}.
\end{proof}

\begin{remark}\label{rem:reg}
    The additional $\tau_g||\boldsymbol{w}_g||_2^2$ terms are redundant regularization terms, potentially impacting the performance of the final model as shown empirically in Section \ref{sec:num_exp}.
 \end{remark}

\begin{theorem}[ADMM-SC Time Complexity]\label{thm:admm_conv}
    Suppose that the $\ell_1$-norm is used in the strongly convex variant of the local model problem \eqref{eq:admm_2d}, and that it is solved via the barrier method with the log barrier function and Newton updates. Then, the ADMM-SC algorithm \ref{alg:admm} equipped with $\rho$ chosen according to Theorem \ref{thm:admm_conv_cri} has an overall worst-case time complexity of $\mathcal{O}\left (\epsilon_1^{-1} \left [(N_{g^{\ast}}+P)^{3.5}\log(\epsilon_2^{-1}) + GP \right ]\right )$ (with $\mathcal{O}(\epsilon_1^{-1})$ communication rounds), where $\epsilon_1$, $\epsilon_2 > 0$ are tolerances on the solutions of ADMM and the strongly convex variant of the local problem \eqref{eq:admm_2d}, respectively, and $N_{g^{\ast}}$ is the greatest number of samples at any client.
\end{theorem}

\begin{proof}
    Proof is provided in Appendix \ref{proof:admm_conv}.
\end{proof}
\section{Numerical Experiments}\label{sec:num_exp}
We discuss numerical experiments that examine the performance of our algorithms and compare them to SOTA baselines. Unless otherwise stated, all numbers reported are averages over 50 repetitions with randomized data shuffling, and all error bars or confidence intervals represent one standard deviation. We use equivalent client weights $\alpha_g$ and local hyperparameters for all clients. Additional experimental details and a \textbf{scalability experiment} are provided in Appendix \ref{app:exp}. All code is available at \href{https://github.com/mibrahim41/FDR-SVM}{this link}.

\textbf{Our Methods.} Our algorithms include: i) \textit{SM}: the FDR-SVM model trained via the SM algorithm in \ref{alg:subgrad} with a diminishing step-size, ii) \textit{ADMM}: the FDR-SVM model trained via the ADMM algorithm in \ref{alg:admm}, and iii) \textit{ADMM-SC}: the FDR-SVM model with modified client objectives according to Theorem \ref{thm:admm_conv_cri}, trained via the ADMM-SC algorithm in \ref{alg:admm}.

\subsection{UCI Data Experiment}\label{sec:real}
This experiment compares the performance of our methods to various SOTA benchmarks. We use $G=4$ clients for all datasets. Performance is measured in terms of F-1 score. 

\textbf{Datasets.} We utilize 7 popular dataset from the UCI repository. For all datasets $70\%$ of the samples are used for training and the remainder is used for testing. 

\textbf{Baselines.} We use the DR SVM model by \cite{2019regularization} as a centralized benchmark model. For federated baselines, we compare to the popular \texttt{FedSGD}, \texttt{FedAvg} \citep{mcmahan2017}, and \texttt{FedProx} \citep{li2020} used to train an $\ell_2$-squared regularized SVM. We also compare to \texttt{FedDRP} \citep{Khanduri2023} used to train a DR-SVM with a KL divergence ambiguity set.

\textbf{Hyperparameters.} We tune the Wasserstein radius $\varepsilon$ and label-flipping cost $\kappa$ for the centralized baseline, and the initial learning rate $\gamma(0)$ and number of rounds $T$ for all federated baselines. We also tune the number of rounds $T$ and the hyperparameters $\rho$ and $\gamma$ for our methods. We utilize 5-fold cross-validation for hyperparameter tuning.

\begin{table*}[ht]
\centering
\caption{F-1 Score Attained by Classification Models on 7 UCI Datasets.}
\label{tab:real_world}
\smaller
\begin{tabular}{lccccccc}
\toprule
Model    & \multicolumn{1}{c}{Banknote} & \multicolumn{1}{c}{BCW} & \multicolumn{1}{c}{CB} & \multicolumn{1}{c}{MM} & \multicolumn{1}{c}{Parkinson's} & \multicolumn{1}{c}{Rice} & \multicolumn{1}{c}{UKM} \\ \midrule
Central (DR-SVM)  &$\mathbf{.950}\pm .011$&$.964 \pm .013$&$.773 \pm .052$& $.792 \pm .017$&$.904 \pm .025$&$\mathbf{.938} \pm .005$& $.845 \pm .027$  \\ \midrule
FedSGD ($\ell_2$-SVM)  &$\mathbf{.950} \pm .011$&$.914 \pm .019$&$.765 \pm .045$&$.624 \pm .161$&$.752 \pm .204$&$.856 \pm .013$& $.808 \pm .019$\\
FedAvg ($\ell_2$-SVM) &$\mathbf{.950} \pm .011$&$.929 \pm .016$&$.788 \pm .051$&$.787 \pm .024$&$.816 \pm .132$&$.936 \pm .006$& $.847 \pm .027$ \\
FedProx ($\ell_2$-SVM) &$\mathbf{.950} \pm .011$&$.929 \pm .019$&$.780 \pm .075$&$.782 \pm .052$&$.735 \pm .210$&$.931 \pm .006$& $.847 \pm .028$ \\ 
FedDRO (KL) &$.945 \pm .011$& $.925 \pm .019$ & $.738 \pm .036$ &$.783 \pm .019$&$.864 \pm .025$& $.859 \pm .009$&$.718 \pm .001$ \\
\midrule
SM (FDR-SVM)  &$.855 \pm .017$&$.957 \pm .015$&$.769 \pm .054$&$.797 \pm .023$&$\mathbf{.920} \pm .023$&$.936 \pm .006$& $.840 \pm .031$    \\
ADMM (FDR-SVM)  &$\mathbf{.950} \pm .011$&$\mathbf{.967} \pm .014$&$\mathbf{.792} \pm .047$&$\mathbf{.798} \pm .017$&$.911 \pm .021$&$\mathbf{.938} \pm .005$& $\mathbf{.848} \pm .027$ \\
ADMM-SC (FDR-SVM)  &$\mathbf{.950} \pm .011$&$.966 \pm .014$&$.765 \pm .048$&$.797 \pm .019$&$.902 \pm .026$&$\mathbf{.938} \pm .005$& $.846 \pm .027$\\
\bottomrule
\end{tabular}
\end{table*}

\textbf{Results.} Table \ref{tab:real_world} presents the performance achieved by each model. Our proposed models consistently outperform the federated benchmark models on most datasets, often by a substantial margin. This underscores the value of DR in modeling uncertainty, and the benefits of using algorithms specifically designed for the FDR-SVM problem. We note that one or more of our FDR-SVM algorithms attains the highest F-1 score for all datasets. Additionally, the ADMM algorithm generally outperforms SM algorithm on most datasets, except for Parkinson's, which suggests that ADMM often converges in practice, in many settings, even if theoretical convergence is not guaranteed. 

We also observe that in some cases, ADMM-SC performs much worse than ADMM (e.g., on BCW and UKM) but can also closely match its performance (e.g., on Banknote, MM, and Rice). This suggests that pursuing guaranteed theoretical convergence comes at the cost of stronger regularization, and thus, potentially weaker performance. One notable observation is that the ADMM or SM algorithms can sometimes outperform the centralized model. This suggests that our proposed MoWB ambiguity set can \textbf{outperform} the classical Wasserstein ball in modeling uncertainty in some settings as hypothesized in Remark \ref{remark:amb_set}. Finally, we note that \texttt{FedAvg} and \texttt{FedProx} failed to consistently converge for the Rice dataset despite extensive hyperparameter tuning and a diminishing learning rate. This suggests a lack of stability potentially due to the non-smoothness of the hinge loss, which further highlights the benefits of our algorithms. We highlight through a one-sided Wilcoxon singed-rank test in Appendix \ref{app:stat_sig} that performance improvements offered by our model are statistically significant.

\subsection{Industrial Data Experiment}\label{sec:sens}
We utilize industrial data from degrading pumps to examine the performance of our models. We explore 5 settings: i) nominal: training data is distributed evenly across clients and classes, ii) client imbalance: training data distribution across clients is $[70\%, 15\%, 10\%, 5\%]$, iii) class imbalance: training data distribution across classes is $[90\%, 10\%]$, iv) client+class imbalance: a combination of the previous two settings, and v) noisy labels: $15\%$ of the training labels are flipped. This experiment contains two distinct components: 1) a sensitivity analysis, and 2) a benchmarking study. Performance is evaluated in terms of mean correct classification rate (mCCR) and F-1 score in the sensitivity analysis and benchmarking study, respectively.

\textbf{Dataset.} We utilize industrial data generated via a physics-driven pump model \citep{Mathworks}. The data contains healthy and leak fault classes. Client heterogeneity is simulated by generating different fault severities per client. We use $G=4$, $P=14$, $N=400$, and $N_{Test}=1000$ test samples. The test set contains 500 healthy samples, and 125 samples from each of the 4 fault severities. 

\subsubsection{Sensitivity Analysis}
\textbf{Baseline.} We compare our models to the the central DR-SVM benchmark by \cite{2019regularization}.

\textbf{Hyperparameters.} In this part of the experiment, we plot each algorithm's performance as a function of its hyperparameters. We examine global and local hyperparameters, and vary each of them separately. The global ones are initial step-size $\gamma$ for the SM algorithm, scale parameter $\rho$ for the ADMM algorithms, and total number of rounds $T$ for all algorithms. The local hyperparameters are the label flipping cost $\kappa_g$, and the local Wasserstein ball radius factor $\beta_g$, where the radius is $\varepsilon_g = \frac{1}{\beta_g}{N_g}$. This is used as a simplifying heuristic to relate the radius to the number of training samples. We also vary the central's $\varepsilon$ and $\kappa$, showing only the best performance as a benchmark line on the plots.

\begin{figure*}[ht]
    \centering
    \includegraphics[width=1\textwidth]{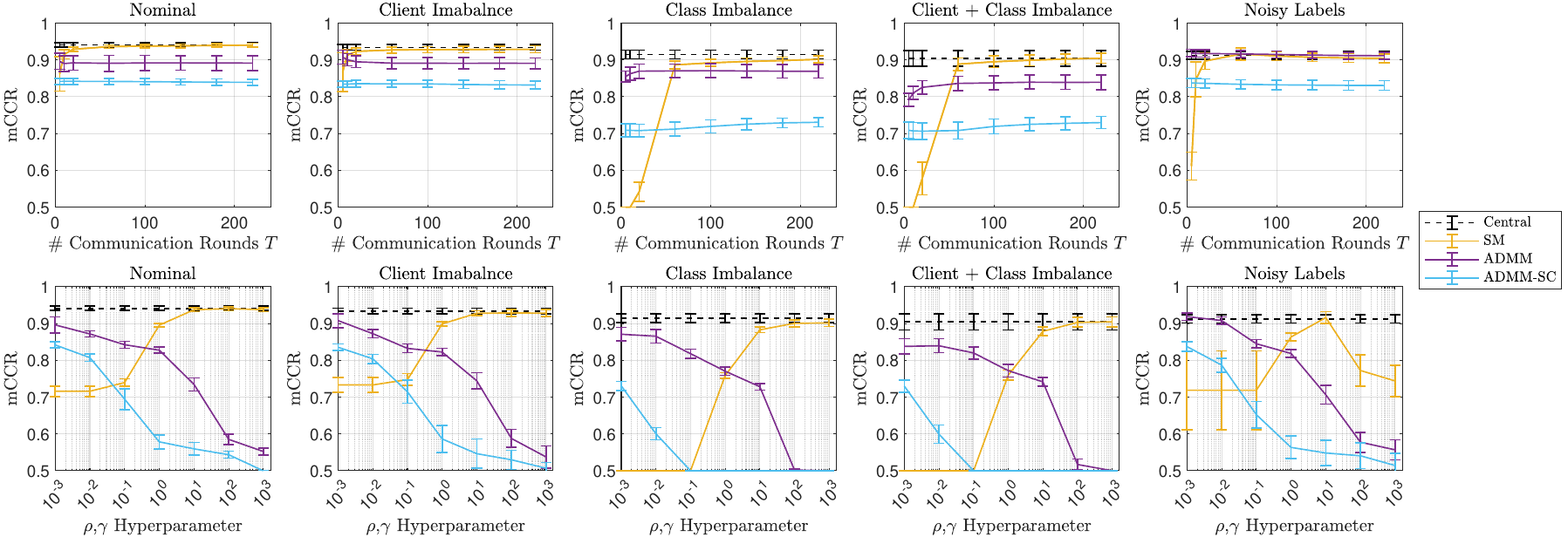}
    \caption{mCCR vs. the Global Hyperparameters, Comparing Our Proposed Methods to the Best-Performing Central Model.}
    \label{fig:global_imb}
\end{figure*}

\begin{figure*}[ht]
    \centering
    \includegraphics[width=1\textwidth]{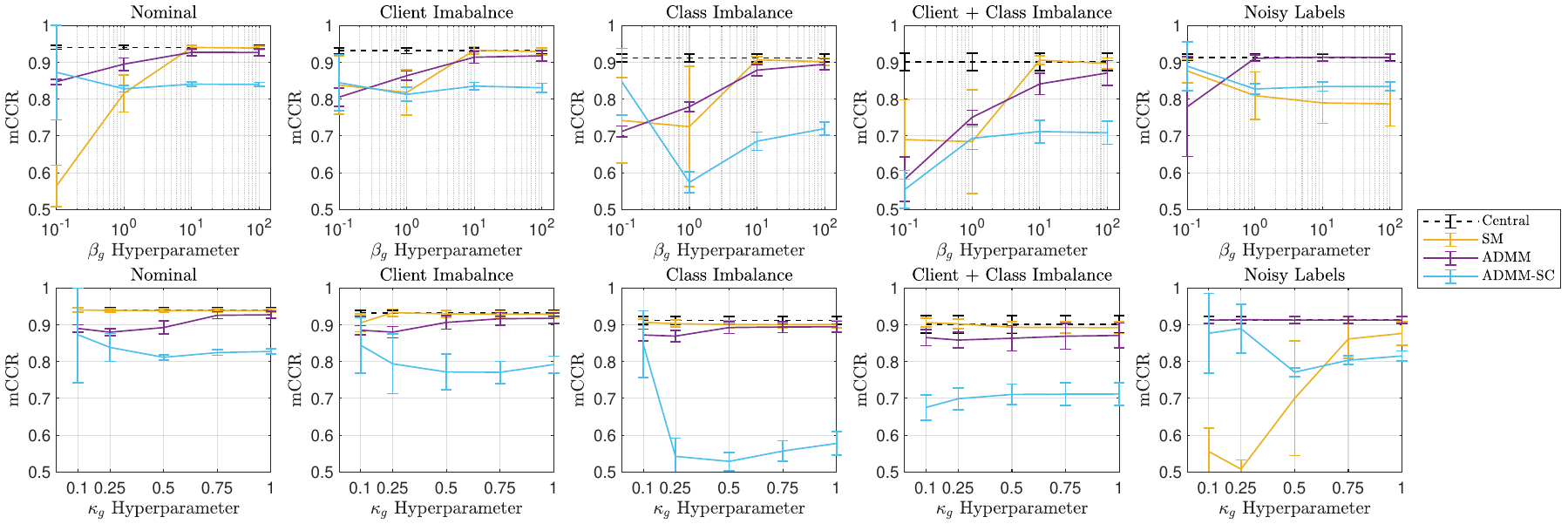}
    \caption{mCCR vs. the Local Hyperparameters, Comparing Our Proposed Methods to the Best-Performing Central Model.}
    \label{fig:local_imb}
\end{figure*}

\textbf{Results.} The \textit{global hyperparameters} effects are shown in Figure \ref{fig:global_imb}. The SM often obtains a higher peak performance in most settings than ADMM, however, it can require more communication rounds to do so. This is highlighted in the `class' imbalance and `client + class' imbalance settings. The SM algorithm is also relatively stable to the choice of $\gamma$, and maintains peak performance across a wide range of values. However, the ADMM algorithm is sensitive to $\rho$, with performance rapidly decreasing as $\rho$ increases. This suggests that ADMM may require more involved global hyperparameter tuning in practice, but can achieve its peak performance in fewer communication rounds. As hypothesized in Remark \ref{rem:reg}, we observe that ADMM largely outperforms ADMM-SC
due to the additional strongly convex regularization terms. Finally, we also observe that SM and ADMM slightly outperform the best-performing central model in the noisy labels case. This can likely be attributed to our novel ambiguity set's improved uncertainty modeling capability.

The \textit{local hyperparameter} effects are shown in Figure \ref{fig:local_imb}. Generally, model performance improves as the radius of the local Wasserstein balls decreases (by increasing $\beta_g$). This suggests that performance degrades with larger local Wasserstein balls due to over-conservatism. However, in noisy labels settings, performance of the SM model deteriorates as the local radius decreases. This suggests the need for larger local ambiguity sets to adequately capture label uncertainty. We also observe that the SM model is highly sensitive to the local radius and $\kappa_g$ in noisy label settings, whereas the ADMM achieves its best performance across a broader range of hyperparameter values. This suggests the need for local hyperparameter fine tuning if SM is used in an application with highly uncertain labels. Moreover, it can be seen that in all other settings, ADMM's performance tends to improve as $\kappa_g$ increases, which is to be expected, since lower $\kappa_g$ implies greater anticipation of label uncertainty, and thus over-conservatism. Similar to our observation in the global hyperparameter experiments, we again observe the suboptimality of the ADMM-SC, which underscores the sacrifice in model performance that is associated with enforcing guaranteed convergence.

\subsubsection{Benchmarking}
\textbf{Baselines.} We utilize the same benchmark models utilized in the UCI data experiment.

\textbf{Hyperparameters.} We tune the same global and local hyperparamaters discussed in the sensitivity analysis of the industrial data experiment. However, we use 5-fold cross-validation for hyperparameter tuning, and we tune both the global and local hyperparameters simultaneously.

\textbf{Results.} Table \ref{tab:industrial} shows the results of this study, which are averaged over 10 repetitions. As in the UCI data experiment, we observe that one of our methods obtains the best performance out of all federated approaches for all the settings tested. This underscores the practical impact of our proposed model and its solution algorithms in federated classification problems. Unlike the UCI data experiment, however, we observe that the SM algorithm is the peak performer in most settings in this experiment. This suggests that algorithm choice should be influenced by the dataset under study among other factors. Finally, we observe that for this dataset ADMM-SC is largely outperformed by ADMM. This again provides an example where opting for theoretically guaranteed convergence may come at a sacrifice in model accuracy due to redundant regularization.

\begin{table*}[ht]
\centering
\caption{F-1 Score Attained by Classification Models on Industrial Dataset in 5 Settings.}
\label{tab:industrial}
\smaller
\begin{tabular}{lccccc}
\toprule
Model    & \multicolumn{1}{c}{Nominal} & \multicolumn{1}{c}{Client Imbalance} & \multicolumn{1}{c}{Class Imbalance} & \multicolumn{1}{c}{Client + Class Imbalance} & \multicolumn{1}{c}{Noisy Labels} \\ \midrule
Central (DR-SVM)  &$.939\pm .004$&$\mathbf{.930} \pm .012$&$\mathbf{.903} \pm .012$& $\mathbf{.901} \pm .017$&$\mathbf{.908} \pm .011$ \\ \midrule
FedSGD ($\ell_2$-SVM)  &$.886 \pm .008$&$.887 \pm .007$&$.675 \pm .014$&$.685 \pm .030$&$.861 \pm .009$\\
FedAvg ($\ell_2$-SVM) &$.923 \pm .006$&$.919 \pm .010$&$.866 \pm .018$&$.845 \pm .059$&$.894 \pm .017$ \\
FedProx ($\ell_2$-SVM) &$.926 \pm .010$&$.919 \pm .011$&$.862 \pm .019$&$.842 \pm .058$&$.894 \pm .019$ \\ 
FedDRO (KL) &$.913 \pm .007$& $.914 \pm .010$ & $.858 \pm .014$ &$.835 \pm .052$&$.883 \pm .012$ \\
\midrule
SM (FDR-SVM)  &$\mathbf{.942} \pm .006$&$\mathbf{.930} \pm .010$&$\mathbf{.883} \pm .022$&$\mathbf{.879} \pm .035$&$.894 \pm .015$    \\
ADMM (FDR-SVM)  &$.918 \pm .010$&$.910 \pm .028$&$.868 \pm .018$&$.855 \pm .025$&$\mathbf{.903} \pm .011$ \\
ADMM-SC (FDR-SVM)  &$.817 \pm .009$&$.819 \pm .011$&$.638 \pm .020$&$.627 \pm .028$&$.806 \pm .013$\\
\bottomrule
\end{tabular}
\end{table*}
\section{Conclusions}\label{sec:conc}
We propose an FDR-SVM--a classifier that is \textit{distributionally robust} to uncertainty in training data features and labels, and can be trained in a \textit{federated} fashion. To that end, we propose a novel MoWB ambiguity set, extending the classical Wasserstein ball to the federated setting. We also demonstrate that it exhibits desirable out-of-sample guarantees, and that it allows for problem separability. We then rigorously derive two different algorithms to train our proposed model and analyze their convergence behavior. Finally, we evaluate the performance of our proposed model using various datasets, demonstrating that it frequently outperforms existing methods. Future extensions could utilize the MoWB to robustify other FL models, or explore convergence behavior of our algorithms with partial client participation.

\begin{acknowledgements} 
We thank the referees and the AC for their valuable input, which allowed us to improve the presentation and content of our paper. 

Michael Ibrahim, Heraldo Rozas, and Nagi Gebraeel were supported by the National Aeronautics and Space Administration (NASA), Space Technology Research Institute (STRI)  Habitats  Optimized  for  Missions  of Exploration (HOME) ‘SmartHab’ Project (grant No. 80NSSC19K1052). Weijun Xie was supported in part by National Science Foundation (NSF) grant No. 2246414 and Office of Naval Research (ONR) grant No. N00014-24-1-2066.
\end{acknowledgements}

\bibliography{references}

\newpage
\onecolumn
\title{FDR-SVM: A Federated Distributionally Robust Support Vector Machine via a Mixture of Wasserstein Balls Ambiguity Set\\(Supplementary Material)}
\maketitle
\appendix
\addtocontents{toc}{\protect\setcounter{tocdepth}{3}}

{\small \tableofcontents}
\newpage

\section{Additional Background on Wasserstein DRO} \label{app:bg}
Distributionally robust optimization has been recently popularized as an intermediate approach between stochastic programming (SP) \citep{shapiro_stochprog} and robust optimization (RO) \citep{bental_ro}. Indeed, it can be viewed as a stochastic programming problem where the true distribution $\mathbb{P}$ governing the data is unknown. Alternatively, it can be seen as a robust optimization problem where worst-case perturbations of the data distribution are modeled rather than those of individual data points. This makes DRO attractive as it is a method of modeling the uncertainty without requiring knowledge of the true distribution $\mathbb{P}$ (like in SP) or potentially being overly conservative (like in RO) \citep{bertsimas2004}. DRO relies on defining an ambiguity set $\mathcal{A}$ of distributions, and subsequently minimizing the worst-case risk attained by any distribution $\mathbb{Q}$ within the ambiguity set $\mathcal{A}$. There have been various different methods of defining the ambiguity set in the literature. This includes moment-based methods \citep{delage2010}, which use certain moment properties to define the set, and distance-based methods \citep{bayraksan2015,kuhn2019wasserstein}, which define the set as a sphere centered at some reference distribution, and whose radius is in the sense of some distance measure. Commonly used measures include $\phi$-divergences (such as KL divergence) \citep{bayraksan2015} and the Wasserstein distance \citep{kuhn2019wasserstein}. Moreover, in most Machine Learning problems, the reference distribution is taken to be the empirical distribution $\widehat{\mathbb{P}}_N$ of the $N$ training data samples.

In our work, we focus on ambiguity sets defined via the type-1 Wasserstein distance. This is because Wasserstein DRO offers many desirable advantages over its counterparts, as demonstrated by \cite{kuhn2019wasserstein}. For example, the Wasserstein ambiguity set can contain both discrete and continuous distributions regardless of the structure of the empirical distribution, which cannot be achieved by the KL divergence ambiguity set. Moreover, one can derive out-of-sample performance guarantees using concentration inequalities when using a Wasserstein ambiguity set, which cannot be achieved in moment-based approaches. The type-1 Wasserstein $W_{d,1}$ distance \citep{kant1958} is commonly referred to as optimal transport metric or earth mover’s distance. This is because of its interpretation as the minimum cost of transforming a distribution $\mathbb{Q}$ to $\mathbb{Q}^{\prime}$. Therefore, it utilizes a transportation cost function $d(\boldsymbol{\xi},\boldsymbol{\xi}^{\prime})$ to define the transportation cost function per unit mass from point $\boldsymbol{\xi}$ to point $\boldsymbol{\xi}^{\prime}$. We can express the type-1 Wasserstein distance mathematically as follows.
\begin{equation*}
    W_{d,1}(\mathbb{Q},\mathbb{Q}') \coloneqq \inf_{\pi \in \Pi(\mathbb{Q},\mathbb{Q}')} \int_{\Xi \times \Xi} d \left( \boldsymbol{\xi}, \boldsymbol{\xi}^{\prime} \right ) \pi(\text{d} \boldsymbol{\xi}, \text{d} \boldsymbol{\xi}^{\prime}),
\end{equation*}
where $d(\boldsymbol{\xi},\boldsymbol{\xi}^{\prime})$ denotes the transportation cost function, and $\Pi(\mathbb{Q},\mathbb{Q}')$ is the set of all joint distributions of $\boldsymbol{\xi}$ and $\boldsymbol{\xi}^{\prime}$ with marginals $\mathbb{Q}$ and $\mathbb{Q}^{\prime}$, respectively. Note that the data in our classification problem is comprised of continuous features $\boldsymbol{x} \in \mathcal{X} \subseteq \mathbb{R}^P$ and categorical labels $y \in \{ -1,+1 \}$. Therefore, a commonly used transportation cost function for such setting is
\begin{equation*}
d(\boldsymbol{\xi},\boldsymbol{\xi}^{\prime}) \coloneqq ||\boldsymbol{x} - \boldsymbol{x}^{\prime}|| + \kappa \mathbbm{1}_{\{y \neq y^{\prime}\}},
\end{equation*}
where $||\cdot||$ is any norm on $\mathbb{R}^P$, and $\kappa$ is the label-flipping cost, treated as a user-defined hyperparameter. This cost function allows us to quantify differences in both the features and labels between samples.

\section{Proofs and Supplementary Theoretical Results}
\subsection{Preliminary Lemmas} \label{sec:lemmas}
\begin{lemma} \label{lem:diff_max}
    Any two real scalars $a,a^{\prime} \in \mathbb{R}$ obey the following
    \begin{equation*}
        |\max\{0,a\} - \max \{0,a^{\prime} \}| \leq |a - a^{\prime}|.
    \end{equation*}
\end{lemma}
\begin{proof}
    To see this, consider the following cases:
    \begin{enumerate}
        \item $a,a^{\prime} \geq 0$. In this case one can directly see that
        \begin{equation*}
            |\max\{0,a\} - \max \{0,a^{\prime} \}| = |a - a^{\prime}|
        \end{equation*}

        \item $a \geq 0, \ a^{\prime} < 0$. In this case, we have the following:
        \begin{equation*}
            |\max\{0,a\} - \max \{0,a^{\prime} \}|  = a  < |a| + |a^{\prime}| = |a - a^{\prime}|.
        \end{equation*}

        \item $a < 0, \ a^{\prime} \geq 0$. This is symmetric to the previous case.

        \item $a < 0, \ a^{\prime} < 0$. In this case we have the following:
        \begin{equation*}
            |\max\{0,a\} - \max \{0,a^{\prime} \}|  = 0 \leq | a - a^{\prime}|.
        \end{equation*}
    \end{enumerate}
\end{proof}
\begin{lemma}[SM Objective Function Convexity]\label{lem:sub_con}
    Suppose Assumption \ref{assump:feat_supp} holds and let $f(\boldsymbol{w}) = \sum_{g=1}^{G} \alpha_g \sup_{\mathbb{Q}_g \in \mathcal{A}^{(g)}_{\varepsilon_g,1,d}(\Xi)}\mathbb{E}^{\mathbb{Q}_g}[\ell_H(\boldsymbol{w};\boldsymbol{\xi})]$. Then, $f(\boldsymbol{w})$ is convex in $\boldsymbol{w}$.
\end{lemma}
\begin{proof}
    Since $\ell_H(\boldsymbol{w},\boldsymbol{\xi})$ is a maximum of linear terms in $\boldsymbol{w}$, then it is convex in $\boldsymbol{w}$. Moreover, sums, scalar multiplication, taking the supremum, and the expectation are all operations that preserve convexity \citep{boyd2004convex}. Thus, $f(\boldsymbol{w})$ is convex in $\boldsymbol{w}$.
\end{proof}

\begin{lemma}[SM Objective Function Lipschitz Continuity] \label{lem:lip_cont}
    Suppose Assumption \ref{assump:feat_supp} holds and let $f(\boldsymbol{w}) = \sum_{g=1}^{G} \alpha_g \sup_{\mathbb{Q}_g \in \mathcal{A}^{(g)}_{\varepsilon_g,1,d}(\Xi)}\mathbb{E}^{\mathbb{Q}_g}[\ell_H(\boldsymbol{w};\boldsymbol{\xi})]$. Then, $f(\boldsymbol{w})$ is Lipschitz continuous in $\boldsymbol{w}$.
\end{lemma}
\begin{proof}
    As discussed by \cite{2019regularization}, if assumption \ref{assump:feat_supp} holds then one can obtain the discrete distribution described in \eqref{eq:ext_dist} that attains the worst case risk. Therefore, we have the following:
    \begin{align*}
    \allowdisplaybreaks
        f(\boldsymbol{w}) &=  \sum_{g=1}^{G} \alpha_g f_g(\boldsymbol{w}) \\
        & \coloneqq \sum_{g=1}^{G} \alpha_g \Bigg (\frac{1}{N_g} \sum_{n_g=1}^{N_g} {\beta_{n_g}^+}^{\ast} \ell_H(\boldsymbol{w};(\widehat{y}_{n_g},\widehat{\boldsymbol{z}}_{n_g}^+)) + {\beta_{n_g}^-}^{\ast} \ell_H(\boldsymbol{w};(-\widehat{y}_{n_g},\widehat{\boldsymbol{z}}_{n_g}^-)) \Bigg ),
    \end{align*}
    Now, suppose we have $\boldsymbol{w}$ and $\boldsymbol{w}^{\prime}$ which correspond to worst-case distributions characterized by $({\beta_{n_g}^{\pm}}^{\ast},\widehat{\boldsymbol{z}}_{N_g}^{\pm})$ and $({\beta_{n_g}^{\pm \prime}}^{\ast},\widehat{\boldsymbol{z}}_{N_g}^{\pm \prime})$, respectively. Then we can write the following:
    \begin{subequations}
    \allowdisplaybreaks
    \begin{align}
    \allowdisplaybreaks
        &|f_g(\boldsymbol{w}) - f_g(\boldsymbol{w}^{\prime})| \nonumber \\
        & \label{eq:lip_1a}=\frac{1}{N_g} \Bigg | \sum_{n_g=1}^{N_g} \bigg [{\beta_{n_g}^{+}}^{\ast} \ell_H(\boldsymbol{w};( \widehat{y}_{n_g},\widehat{\boldsymbol{z}}_{n_g}^{+})) + {\beta_{n_g}^{-}}^{\ast} \ell_H(\boldsymbol{w};( -\widehat{y}_{n_g},\widehat{\boldsymbol{z}}_{n_g}^{-})) \bigg ] \nonumber \\
        & \qquad \qquad \qquad - \sum_{n_g=1}^{N_g} \bigg [{\beta_{n_g}^{+ \prime}}^{\ast} \ell_H(\boldsymbol{w}^{\prime};( \widehat{y}_{n_g},\widehat{\boldsymbol{z}}_{n_g}^{+ \prime})) + {\beta_{n_g}^{- \prime}}^{\ast} \ell_H(\boldsymbol{w}^{\prime};( -\widehat{y}_{n_g},\widehat{\boldsymbol{z}}_{n_g}^{- \prime})) \bigg] \Bigg |\\
        &\label{eq:lip_1c} \leq \frac{1}{N_g} \Bigg | \sum_{n_g=1}^{N_g} \bigg [{\beta_{n_g}^{+}}^{\ast} \ell_H(\boldsymbol{w};( \widehat{y}_{n_g},\widehat{\boldsymbol{z}}_{n_g}^{+})) - {\beta_{n_g}^{+ \prime}}^{\ast} \ell_H(\boldsymbol{w}^{\prime};( \widehat{y}_{n_g},\widehat{\boldsymbol{z}}_{n_g}^{+ \prime})) \bigg ] \Bigg| \nonumber \\
        & \qquad \qquad \qquad + \Bigg | \sum_{n_g=1}^{N_g} \bigg [{\beta_{n_g}^{-}}^{\ast} \ell_H(\boldsymbol{w};( -\widehat{y}_{n_g},\widehat{\boldsymbol{z}}_{n_g}^{-})) - {\beta_{n_g}^{- \prime}}^{\ast} \ell_H(\boldsymbol{w}^{\prime};( -\widehat{y}_{n_g},\widehat{\boldsymbol{z}}_{n_g}^{- \prime})) \bigg] \Bigg |\\
        &\label{eq:lip_1d} \leq \frac{1}{N_g} \sum_{n_g=1}^{N_g} \Bigg [\bigg | {\beta_{n_g}^{+}}^{\ast} \ell_H(\boldsymbol{w};( \widehat{y}_{n_g},\widehat{\boldsymbol{z}}_{n_g}^{+})) - {\beta_{n_g}^{+ \prime}}^{\ast} \ell_H(\boldsymbol{w}^{\prime};( \widehat{y}_{n_g},\widehat{\boldsymbol{z}}_{n_g}^{+ \prime}))  \bigg|  \nonumber \\
        & \qquad \qquad \qquad + \bigg|{\beta_{n_g}^{-}}^{\ast} \ell_H(\boldsymbol{w};( -\widehat{y}_{n_g},\widehat{\boldsymbol{z}}_{n_g}^{-})) - {\beta_{n_g}^{- \prime}}^{\ast} \ell_H(\boldsymbol{w}^{\prime};( -\widehat{y}_{n_g},\widehat{\boldsymbol{z}}_{n_g}^{- \prime})) \bigg | \Bigg ]\\
        &\label{eq:lip_1e} \leq \frac{1}{N_g} \sum_{n_g=1}^{N_g} \Bigg [\bigg | {\beta_{n_g}^{+}}^{\ast} \ell_H(\boldsymbol{w};( \widehat{y}_{n_g},\widehat{\boldsymbol{z}}_{n_g}^{+})) - {\beta_{n_g}^{+}}^{\ast} \ell_H(\boldsymbol{w}^{\prime};( \widehat{y}_{n_g},\widehat{\boldsymbol{z}}_{n_g}^{+}))  \bigg| \nonumber \\
        & \qquad \qquad \qquad + \bigg|{\beta_{n_g}^{-}}^{\ast} \ell_H(\boldsymbol{w};( -\widehat{y}_{n_g},\widehat{\boldsymbol{z}}_{n_g}^{-})) - {\beta_{n_g}^{- }}^{\ast} \ell_H(\boldsymbol{w}^{\prime};( -\widehat{y}_{n_g},\widehat{\boldsymbol{z}}_{n_g}^{-})) \bigg | \Bigg ]\\
        &\label{eq:lip_1f} \leq \frac{1}{N_g} \sum_{n_g=1}^{N_g} \Bigg [\bigg | {\beta_{n_g}^{+}}^{\ast} \max\{0,1-\widehat{y}_{n_g} \cdot \boldsymbol{w}^{\mathsf{T}}\widehat{\boldsymbol{z}}_{n_g}^{+} \} - {\beta_{n_g}^{+}}^{\ast} \max\{0,1-\widehat{y}_{n_g} \cdot \boldsymbol{w}^{\prime{\mathsf{T}}}\widehat{\boldsymbol{z}}_{n_g}^{+} \} \bigg| \nonumber \\
        & \qquad \qquad \qquad +  \bigg|{\beta_{n_g}^{-}}^{\ast} \max\{0,1+\widehat{y}_{n_g} \cdot \boldsymbol{w}^{\mathsf{T}}\widehat{\boldsymbol{z}}_{n_g}^{-} \} - {\beta_{n_g}^{- }}^{\ast} \max\{0,1+\widehat{y}_{n_g} \cdot \boldsymbol{w}^{\prime{\mathsf{T}}}\widehat{\boldsymbol{z}}_{n_g}^{-} \} \bigg | \Bigg ]\\
        &\label{eq:lip_1g} = \frac{1}{N_g} \sum_{n_g=1}^{N_g} \Bigg [\bigg | (\boldsymbol{w}^{\prime} - \boldsymbol{w})^{\mathsf{T}}({\beta_{n_g}^{+}}^{\ast}\cdot \widehat{y}_{n_g} \cdot \widehat{\boldsymbol{z}}_{n_g}^{+}) \bigg | + \bigg | (\boldsymbol{w} - \boldsymbol{w}^{\prime})^{\mathsf{T}}({\beta_{n_g}^{-}}^{\ast}\cdot \widehat{y}_{n_g} \cdot \widehat{\boldsymbol{z}}_{n_g}^{-})\bigg | \Bigg ]\\
        &\label{eq:lip_1h} \leq \frac{1}{N_g} \sum_{n_g=1}^{N_g} ||\boldsymbol{w} - \boldsymbol{w}^{\prime}|| \Bigg (\left | \left | {\beta_{n_g}^{+}}^{\ast}\cdot \widehat{y}_{n_g} \cdot \widehat{\boldsymbol{z}}_{n_g}^{+} \right | \right |_{\ast} + \left | \left | {\beta_{n_g}^{-}}^{\ast}\cdot \widehat{y}_{n_g} \cdot \widehat{\boldsymbol{z}}_{n_g}^{-} \right | \right |_{\ast} \Bigg )\\
         &\label{eq:lip_1i} = ||\boldsymbol{w} - \boldsymbol{w}^{\prime}|| \Bigg [ \frac{1}{N_g} \sum_{n_g=1}^{N_g}\left | \left | {\beta_{n_g}^{+}}^{\ast}\cdot \widehat{y}_{n_g} \cdot \widehat{\boldsymbol{z}}_{n_g}^{+} \right | \right |_{\ast} + \left | \left | {\beta_{n_g}^{-}}^{\ast}\cdot \widehat{y}_{n_g} \cdot \widehat{\boldsymbol{z}}_{n_g}^{-} \right | \right |_{\ast}\Bigg],
    \end{align}
    \end{subequations}
    where \eqref{eq:lip_1c} and \eqref{eq:lip_1d} follow from the triangle inequality, and \eqref{eq:lip_1e} follows by noting that the distribution characterized by $({\beta_{n_g}^{\pm \prime}}^{\ast},\widehat{\boldsymbol{z}}_{N_g}^{\pm \prime})$ maximizes the expected risk with respect to $\boldsymbol{w}^{\prime}$, thus the distribution characterized by $({\beta_{n_g}^{\pm }}^{\ast},\widehat{\boldsymbol{z}}_{N_g}^{\pm})$ will at most attain the same risk with respect to $\boldsymbol{w}^{\prime}$. Additionally, \eqref{eq:lip_1f} follows from the definition of the hinge loss function, \eqref{eq:lip_1g} follows from Lemma \ref{lem:diff_max}, and \eqref{eq:lip_1h} follows from the Cauchy-Schwarz inequality, where $||\cdot||_{\ast}$ is the dual norm of $||\cdot||$ used to measure distances in the space of $\boldsymbol{w}$. Given the previous, we can obtain the final result as follows:
    \begin{subequations}
    \begin{align}
        |f(\boldsymbol{w}) - f(\boldsymbol{w}^{\prime})| & = \left |\sum_{g=1}^{G} \alpha_g f_g(\boldsymbol{w}) - \sum_{g=1}^{G} \alpha_g f_g(\boldsymbol{w}^{\prime}) \right |\\
        &\label{eq:lip_2b} \leq \sum_{g=1}^{G} \alpha_g |f_g(\boldsymbol{w}) - f_g(\boldsymbol{w}^{\prime})| \\ 
        &\label{eq:lip_2c} \leq || \boldsymbol{w} - \boldsymbol{w}^{\prime}|| \sum_{g=1}^G \alpha_g \Lip(f_g(\boldsymbol{w})),
    \end{align}
    \end{subequations}
    where \eqref{eq:lip_2b} follows from the triangle inequality and $\Lip(f_g(\boldsymbol{w}))$ is taken from \eqref{eq:lip_1i}.
\end{proof}

\begin{lemma}[SM Objective Function Coercivity] \label{lem:coer}
    Suppose Assumption \ref{assump:feat_supp} holds and let $f(\boldsymbol{w}) = \sum_{g=1}^{G} \alpha_g \sup_{\mathbb{Q}_g \in \mathcal{A}^{(g)}_{\varepsilon_g,1,d}(\Xi)}\mathbb{E}^{\mathbb{Q}_g}[\ell_H(\boldsymbol{w};\boldsymbol{\xi})]$. Then, $f(\boldsymbol{w})$ is coercive in $\boldsymbol{w}$.
\end{lemma}
\begin{proof}
    We begin out proof by studying each of the individual terms $f_g(\boldsymbol{w})$ as follows
    \begin{subequations}
    \allowdisplaybreaks
    \begin{align}
        \allowdisplaybreaks
        \nonumber f_g(\boldsymbol{w}) &\coloneq \sup_{\mathbb{Q}_g \in \mathcal{A}^{(g)}_{\varepsilon_g,1,d}(\Xi)}\mathbb{E}^{\mathbb{Q}_g}[\ell_H(\boldsymbol{w};\boldsymbol{\xi})]\\
        &\label{eq:coer1} = \inf_{\lambda_g \geq 0} \lambda_g \varepsilon_g +  \frac{1}{N_g} \sum_{n_g=1}^{N_g} \sup_{\boldsymbol{\xi} \in \Xi} \left \{\ell_H(\boldsymbol{w};\boldsymbol{\xi}) - \lambda_g d(\boldsymbol{\xi},\widehat{\boldsymbol{\xi}}_{n_g}) \right \}\\
        &\label{eq:coer2} = \left \{ \begin{aligned}
            &\inf_{\lambda_g \geq 0, s_{n_g}} && \lambda_g \varepsilon_g +  \frac{1}{N_g} \sum_{n_g=1}^{N_g} s_{n_g} &&\\
            & \subjectto && \sup_{\boldsymbol{\xi} \in \Xi} \left \{\ell_H(\boldsymbol{w};\boldsymbol{\xi}) - \lambda_g d(\boldsymbol{\xi},\widehat{\boldsymbol{\xi}}_{n_g}) \right \} \leq s_{n_g} && \forall n_g \in [N_g]\\
        \end{aligned} \right.\\
        &\label{eq:coer3} = \left \{ \begin{aligned}
            &\inf_{\lambda_g \geq 0, s_{n_g}} && \lambda_g \varepsilon_g +  \frac{1}{N_g} \sum_{n_g=1}^{N_g} s_{n_g} &&\\
            & \subjectto && \sup_{\boldsymbol{x} \in \mathcal{X}} \left \{\ell_H(\boldsymbol{w};(\boldsymbol{x},\widehat{y}_{n_g})) - \lambda_g ||\boldsymbol{x} - \widehat{\boldsymbol{x}}_{n_g}|| \right \} \leq s_{n_g} && \forall n_g \in [N_g]\\
            &&& \sup_{\boldsymbol{x} \in \mathcal{X}} \left \{\ell_H(\boldsymbol{w};(\boldsymbol{x},-\widehat{y}_{n_g})) - \lambda_g ||\boldsymbol{x} - \widehat{\boldsymbol{x}}_{n_g}|| \right \} - \kappa \lambda_g \leq s_{n_g} && \forall n_g \in [N_g]\\
        \end{aligned} \right.\\
        &\label{eq:coer4} = \left \{ \begin{aligned}
            &\inf_{\lambda_g, s_{n_g}} && \lambda_g \varepsilon_g +  \frac{1}{N_g} \sum_{n_g=1}^{N_g} s_{n_g} &&\\
            & \subjectto && \ell_H(\boldsymbol{w};(\widehat{\boldsymbol{x}}_{n_g},\widehat{y}_{n_g})) \leq s_{n_g} && \forall n_g \in [N_g]\\
            &&& \ell_H(\boldsymbol{w};(\widehat{\boldsymbol{x}}_{n_g},-\widehat{y}_{n_g})) - \kappa \lambda_g \leq s_{n_g} && \forall n_g \in [N_g]\\
            &&& \lambda_g \geq ||\boldsymbol{w}||_{\ast}
        \end{aligned} \right.
    \end{align}
    \end{subequations}
    where \eqref{eq:coer1} follows from the strong duality result presented by \cite{2019regularization,kuhn2019wasserstein}, \eqref{eq:coer2} is obtained through the introduction of slack variables and moving the maximization problems to the constraints, and \eqref{eq:coer3} is obtained through the definition of the separable transportation cost function \eqref{eq:cost_func} and by noting that $y \in \{ -1, +1\}$, and finally \eqref{eq:coer4} is obtained by recalling that the hinge loss function $\ell_H(\boldsymbol{w};\boldsymbol{\xi})$ is convex and Lipschitz continuous in $\boldsymbol{x}$, and therefore it follows from Lemma A.3 in \citep{2019regularization} that
    \begin{equation*}
        \sup_{\boldsymbol{x} \in \mathcal{X}} \left \{\ell_H(\boldsymbol{w};(\boldsymbol{x},y)) - \lambda_g ||\boldsymbol{x} - \widehat{\boldsymbol{x}}|| \right \} = \left \{ \begin{aligned}
            &\ell_H(\boldsymbol{w};(\widehat{\boldsymbol{x}},y)) && \text{if} ||\boldsymbol{w}||_{\ast} \leq \lambda_g\\
            & + \infty && \text{otherwise},
        \end{aligned} \right.
    \end{equation*}
    where $||\cdot||_{\ast}$ is the dual to the norm utilized in the definition of the transportation cost function \eqref{eq:cost_func}. Therefore, as $||\boldsymbol{w}||_{\ast} \rightarrow \infty$, we get that $\lambda_g \rightarrow \infty$. Since $\lambda_g$ has a positive sign in the objective function of \eqref{eq:coer4}, then $f_g(\boldsymbol{w}) \rightarrow +\infty$ as $\lambda_g \rightarrow \infty$. This implies that $f(\boldsymbol{w}) = \sum_g^G f_g(\boldsymbol{w})$ is a coercive function of $\boldsymbol{w}$, since $f(\boldsymbol{w}) \rightarrow +\infty$ as $||\boldsymbol{w}||_{\ast} \rightarrow \infty$.
\end{proof}

\begin{lemma}[ADMM Objective Properties] \label{lem:clo_pro_con}
     Let $f(\boldsymbol{w}_g) = \sup_{\mathbb{Q}_g \in \mathcal{A}^{(g)}_{\varepsilon_g,1,d}(\Xi)}\mathbb{E}^{\mathbb{Q}_g}[\ell_H(\boldsymbol{w}_g;\boldsymbol{\xi})]$, then $f(\boldsymbol{w}_g)$ is a closed proper convex function in $\boldsymbol{w}_g$.
\end{lemma}

\begin{proof}
    Recall that $\ell_H(\boldsymbol{w}_g,\boldsymbol{\xi})$ is convex in $\boldsymbol{w}_g$, and taking the supremum and expectation are operations that preserve convexity \citep{boyd2004convex}, thus $f(\boldsymbol{w}_g)$ is convex in $\boldsymbol{w}_g$. Now, note that
    \begin{equation*}\ell_H(\boldsymbol{w}_g,\boldsymbol{\xi}) \geq 0  \Rightarrow f(\boldsymbol{w}_g) \geq 0  \quad \forall \boldsymbol{w}_g \in \mathbb{R}^P.
    \end{equation*}
    Now, observe that $f(\boldsymbol{0}) = 1$ since $\ell_H(\boldsymbol{0},\boldsymbol{\xi}) = 1 \ \forall \boldsymbol{\xi} \in \Xi$. Since $f(\boldsymbol{w}_g)>-\infty$ and it has a nonempty effective domain, then it is proper convex \citep{cd2006}. Finally, since $f(\boldsymbol{w}_g):\mathbb{R}^P \rightarrow (-\infty,\infty]$ is proper convex, then it is continuous by Proposition 1.3.11 in \citep{bertsekas2009convex}. This implies the closedness of the function.
\end{proof}

\subsection{Proofs of Theoretical Results} \label{sec:proofs}
\subsubsection{Proof of Proposition \ref{lemma:general_form}} \label{proof:general_form}
\begin{proof}
    \begin{subequations}
    \begin{align}
        \inf_{\boldsymbol{w}} \sup_{\mathbb{Q} \in \mathcal{A}_G} \mathbb{E}^{\mathbb{Q}}[\ell_H(\boldsymbol{w};\boldsymbol{\xi})] \label{eq:gen_1a}&= \inf_{\boldsymbol{w}} \sup_{\left \{ \mathbb{Q}_g \in \mathcal{A}^{(g)}_{\varepsilon_g,1,d}(\Xi) \right \}_{g=1}^{G}} \mathbb{E}^{\sum_{g=1}^{G}\alpha_g \mathbb{Q}_g}[\ell_H(\boldsymbol{w};\boldsymbol{\xi})] \\
        &\label{eq:gen_1b} = \inf_{\boldsymbol{w}} \sup_{\left \{ \mathbb{Q}_g \in \mathcal{A}^{(g)}_{\varepsilon_g,1,d}(\Xi)\right \}_{g=1}^{G}} \sum_{g=1}^{G} \alpha_g\mathbb{E}^{\mathbb{Q}_g}[\ell_H(\boldsymbol{w};\boldsymbol{\xi})]\\
        &\label{eq:gen_1c} = \inf_{\boldsymbol{w}} \sum_{g=1}^{G} \alpha_g \sup_{\mathbb{Q}_g \in \mathcal{A}^{(g)}_{\varepsilon_g,1,d}(\Xi)}\mathbb{E}^{\mathbb{Q}_g}[\ell_H(\boldsymbol{w};\boldsymbol{\xi})],
    \end{align}
    \end{subequations}
    where \eqref{eq:gen_1a} follows from the definition of the global ambiguity $\mathcal{A}_G$ set in \eqref{eq:global_amb_set}, \eqref{eq:gen_1b} follows from the Law of Total Expectation, and \eqref{eq:gen_1c} follows by recognizing that the maximization problems are separable due to each decision variable only affecting its corresponding term.
\end{proof}

\subsubsection{Proof of Proposition \ref{prop:oos}} \label{proof:oos}
\begin{proof}
    Suppose the assumptions in the Proposition statement hold. Then, as demonstrated by \cite{kuhn2019wasserstein} we have the following
    \begin{equation*}
        \mathbb{P}^{N_g} \{ \mathbb{P}_g \in \mathcal{A}^{(g)}_{\varepsilon_g,1,d}(\Xi) \} \geq (1 - \eta_g)
    \end{equation*}
    Therefore, we can obtain the following.
    \begin{subequations}
    \begin{align}
        \mathbb{P}^N\{\mathbb{P} \in \mathcal{A}_G\} &\label{eq:oos1} \geq \prod_{g=1}^{G}\mathbb{P}^{N_g} \{ \mathbb{P}_g \in \mathcal{A}^{(g)}_{\varepsilon_g,1,d}(\Xi) \} \\
        &\label{eq:oos2} \geq \prod_{g=1}^{G} (1 - \eta_g),
    \end{align}
    \end{subequations}
    where \eqref{eq:oos1} follows by noting that the local data and Wasserstein balls at all $G$ clients are mutually independent, and that $\mathbb{P} = \sum_{g=1}^{G} \alpha_g \mathbb{P}_g$. Furthermore, note that \eqref{eq:oos1} contains an inequality instead of an equality as there is no guarantee that $\mathbb{P}$ cannot be constructed as a mixture of distributions from the local Wasserstein balls $\{\mathcal{A}^{(g)}_{\varepsilon_g,1,d}(\Xi)\}_{g=1}^G$. Therefore, we have that
    \begin{equation*}
        \mathbb{P}^N\left \{\mathbb{P} \in \mathcal{A}_G \cap \mathbb{P}_g \notin \{\mathcal{A}^{(g)}_{\varepsilon_g,1,d}(\Xi)\}_{g=1}^G \right \} \neq 0.
    \end{equation*}
\end{proof}

\subsubsection{Proof of Proposition \ref{prop:subgrad_comp}} \label{proof:subgrad_comp}
\begin{proof}
    Firstly let us note the following:
    \begin{subequations}
    \begin{align}
        \partial \sup_{\mathbb{Q}_g \in \mathcal{A}^{(g)}_{\varepsilon_g,1,d}(\Xi)}\mathbb{E}^{\mathbb{Q}_g}[\ell_H(\boldsymbol{w};\boldsymbol{\xi})] &\label{eq:subgrad_1a}\supseteq \partial \mathbb{E}^{\mathbb{Q}_g^{\ast}}[\ell_H(\boldsymbol{w};\boldsymbol{\xi})]\\
        &\label{eq:subgrad_1b}= \mathbb{E}^{\mathbb{Q}_g^{\ast}}[\partial \ell_H(\boldsymbol{w};\boldsymbol{\xi})],
    \end{align}
    \end{subequations}
    where \ref{eq:subgrad_1a} follows from Lemma 4.4.1 in \citep{convextextbook} by the fact that $\mathbb{Q}_g^{\ast}$ is a maximizer of the supremum on the left hand side, and \ref{eq:subgrad_1a} follows from the fact that $\ell_H(\boldsymbol{w};\boldsymbol{\xi})$ is convex and integrable, and $\mathbb{Q}_g^{\ast}$ is a discrete distribution. Thus $\mathbb{E}^{\mathbb{Q}_g^{\ast}}[\cdot]$ is a weighted sum.
    
    Now, let us introduce the functions $h_1(\boldsymbol{w})$, and $h_2(\boldsymbol{w})$ to simplify notation as follows:
        \begin{subequations}
        \begin{align}
            \mathbb{E}^{\mathbb{Q}_g^{\ast}} [\ell_H(\boldsymbol{w};\boldsymbol{\xi})]
            &\label{eq:subgrad_2a}= \frac{1}{N_g} \sum_{n_g=1}^{N_g} {\beta_{n_g}^+}^{\ast}\ell_H(\boldsymbol{w};(\widehat{\boldsymbol{z}}_{n_g}^+,\widehat{y}_{n_g})) + {\beta_{n_g}^-}^{\ast}\ell_H(\boldsymbol{w};(\widehat{\boldsymbol{z}}_{n_g}^-,-\widehat{y}_{n_g}))\\
            &\label{eq:subgrad_2c} \coloneqq \frac{1}{N_g} \sum_{n_g=1}^{N_g} h_1(\boldsymbol{w}) + h_2(\boldsymbol{w}),
        \end{align}
        \end{subequations}
        where \ref{eq:subgrad_2a} uses the definition of $\mathbb{Q}_g^{\ast}$ from Equation \ref{eq:ext_dist}, and $\widehat{\boldsymbol{z}}_{n_g}^{\pm} = \widehat{\boldsymbol{x}}_{n_g}-{\boldsymbol{q}_{n_g}^{\pm}}^{\ast}/{\beta_{n_g}^{\pm}}^{\ast}$. Now, observe that we can write the subdifferentials of $h_1(\boldsymbol{w})$ and $h_2(\boldsymbol{w})$ with respect to $\boldsymbol{w}$ as follows:
        \begin{equation*}
            \partial h_1(\boldsymbol{w}) = \left \{ \begin{aligned}
                & \boldsymbol{0} && \text{if } 1 - \widehat{y}_{n_g} \cdot \boldsymbol{w}^{\mathsf{T}} \widehat{\boldsymbol{z}}_{n_g}^+ < 0\\
                & -{\beta_{n_g}^+}^{\ast} \widehat{y}_{n_g}\widehat{\boldsymbol{z}}_{n_g}^+&& \text{if } 1 - \widehat{y}_{n_g} \cdot \boldsymbol{w}^{\mathsf{T}} \widehat{\boldsymbol{z}}_{n_g}^+ > 0\\
                & \text{conv}\left ( \{ \boldsymbol{0}, -{\beta_{n_g}^+}^{\ast} \widehat{y}_{n_g}\widehat{\boldsymbol{z}}_{n_g}^+ \} \right ) && \text{if } 1 - \widehat{y}_{n_g} \cdot \boldsymbol{w}^{\mathsf{T}} \widehat{\boldsymbol{z}}_{n_g}^+ = 0
            \end{aligned} \right.
        \end{equation*}
    
        \begin{equation*}
            \partial h_2(\boldsymbol{w}) = \left \{ \begin{aligned}
                & \boldsymbol{0} && \text{if } 1 + \widehat{y}_{n_g} \cdot \boldsymbol{w}^{\mathsf{T}} \widehat{\boldsymbol{z}}_{n_g}^- < 0\\
                & {\beta_{n_g}^-}^{\ast} \widehat{y}_{n_g}\widehat{\boldsymbol{z}}_{n_g}^-&& \text{if } 1 + \widehat{y}_{n_g} \cdot \boldsymbol{w}^{\mathsf{T}} \widehat{\boldsymbol{z}}_{n_g}^- > 0\\
                & \text{conv}\left ( \{ \boldsymbol{0}, {\beta_{n_g}^-}^{\ast} \widehat{y}_{n_g}\widehat{\boldsymbol{z}}_{n_g}^- \} \right ) && \text{if } 1 + \widehat{y}_{n_g} \cdot \boldsymbol{w}^{\mathsf{T}} \widehat{\boldsymbol{z}}_{n_g}^- = 0
            \end{aligned} \right.
        \end{equation*}
        Therefore, we can use the previous result to obtain the following:
        \begin{equation*}
            \mathbb{E}^{\mathbb{Q}_g^{\ast}}[\partial \ell_H(\boldsymbol{w};\boldsymbol{\xi})] = \frac{1}{N_g} \sum_{n_g=1}^{N_g} \partial h_1(\boldsymbol{w}) + \partial h_2(\boldsymbol{w}),
        \end{equation*}
        where we use the Minkowski sum in the above equation.
\end{proof}

\subsubsection{Proof of Theorem \ref{thm:sm_conv}} \label{proof:sm_conv}
\begin{proof}
    As shown by \cite{nesterov2013}, the subgradient method guarantees convergence assuming the following conditions are met.
    \begin{enumerate}
        \item The objective function is convex.
        \item The objective function is Lipschitz continuous.
        \item The step-size diminishes at an appropriate rate as stated in the theorem statement.
        \item The distance between any optimal solution $\boldsymbol{w}^{\ast}$ and any initial solution $\boldsymbol{w}^{(0)}$ is bounded from above. That is $||\boldsymbol{w}^{\ast} - \boldsymbol{w}^{(0)}|| \leq C$, where $C\in \mathbb{R}$ need not be known.
    \end{enumerate}
    Note that we verify the first two conditions in Lemmas \ref{lem:sub_con} and \ref{lem:lip_cont}, whereas the third condition can be ensured by selecting an appropriately diminishing step-size sequence, as exemplified in the theorem statement. In examining the fourth condition, we note that it is readily satisfied through the coercivity of the objective function, which we prove in Lemma \ref{lem:coer}. To see this, first note that $f(\boldsymbol{0}) = 1$, and by definition $\inf_{\boldsymbol{w}} f(\boldsymbol{w}) \leq f(\boldsymbol{0})$. Suppose we have a set $\mathcal{W} = \{\boldsymbol{w}:f(\boldsymbol{w}) \leq f(\boldsymbol{0}) \}$. We know that for $\boldsymbol{w}^{\ast}$ to be a minimizer of $f(\boldsymbol{w})$, it must be that $\boldsymbol{w}^{\ast} \in \mathcal{W}$. Suppose further that the set $\mathcal{W}$ contains a sequence $\boldsymbol{w}_i$ such that $||\boldsymbol{w}_i|| \rightarrow \infty$. This results in a contradiction, as
    \begin{equation*}
        ||\boldsymbol{w}_i|| \rightarrow \infty \Rightarrow f(\boldsymbol{w}_i) \rightarrow +\infty \Rightarrow \boldsymbol{w}_i \notin \mathcal{W},
    \end{equation*}
    which follows from the coercivity of $f(\boldsymbol{w})$. Thus, there must exist some constant $R \in \mathbb{R}$ such that 
    \begin{equation*}
        \boldsymbol{w} \in \mathcal{W} \Rightarrow ||\boldsymbol{w}|| \leq R.
    \end{equation*}
    Finally, suppose we choose any finite initializer $\boldsymbol{w}^{(0)}$ for the SM algorithm. Then, by the triangle inequality we have
    \begin{equation*}
        ||\boldsymbol{w}^{\ast} - \boldsymbol{w}^{(0)}|| \leq R + ||\boldsymbol{w}^{(0)}||,
    \end{equation*}
    proving that the distance between any initializer $\boldsymbol{w}^{(0)}$ and any optimizer $\boldsymbol{w}^{\ast}$ is indeed bounded from above.
\end{proof}

\subsubsection{Proof of Theorem \ref{thm:subgrad_conv}} \label{proof:subgrad_conv}
\begin{proof}
    We first examine the time complexity of problem \eqref{eq:subgrad_client_prob} that each client $g$ solves at each iteration $t$. When the $\ell_{\infty}$-norm is used in \eqref{eq:subgrad_client_prob}, the problem becomes a Linear Program (LP) with $4N_gP + 2N_g$ decision variables (including slack variables) and $4N_gP + 7N_g$ constraints, where $N_g$ is the number of training samples at the $g^{th}$ client. Solving the problem via the barrier method with the log barrier function and Newton updates requires $\mathcal{O}(\sqrt(C)\log(\epsilon_2^{-1}))$ iterations to reach an $\epsilon_2$-solution \citep{nn1994}, where $C$ is the number of constraints. Moreover, each iteration has an arithmetic complexity of $\mathcal{O}(CD^2)$, where $D$ is the number of decision variables. Therefore, the theoretical worst-case time complexity of solving the problem in \eqref{eq:subgrad_client_prob} is:
    \begin{equation*}
    \mathcal{O}([4N_gP+7N_g]^{1.5}[4N_gP+2N_g]^{2}\log(\epsilon_2^{-1})).
    \end{equation*}
    By eliminating scalar multipliers and constants, we arrive at the following  simplified expression, 
    \begin{equation*}
    \mathcal{O}([N_gP]^{3.5}\log(\epsilon_2^{-1})).
    \end{equation*}
    
    Since all clients can solve their local problems in parallel, and will have the same number of features. Thus, the client with the largest number of samples $N_{g^{\ast}}$ will have the highest time complexity. Furthermore, the central server performs a summation of $G+1$ vectors of dimension $P$ during each iteration $t$, the time complexity of which is $\mathcal{O}(GP)$. We obtain the final result by noting that the subgradient method converges to a solution with tolerance $\epsilon_1$ in $\mathcal{O}(\epsilon_1^{-2})$ iterations \citep{bubeck2015}. Note that we do not explicitly consider the time complexity of computing the local subgradient at each client since it is lower than that of solving the problem in \eqref{eq:subgrad_client_prob}.
\end{proof}

\subsubsection{Proof of Proposition \ref{prop:admm_client_up}} \label{proof:admm_client_up}
In order to obtain updated local model $\boldsymbol{w}_g^{\ast}$, each client $g$ must minimize the global Lagrangian with respect to $\boldsymbol{w}_g$. Thus, the updated local model $\boldsymbol{w}_g^{\ast}$ can be obtained as the minimizer to the following problem.
\begin{subequations}
\allowdisplaybreaks
\begin{align}
    \nonumber J_g (\boldsymbol{w},\boldsymbol{\mu}_g) &=  \inf_{\boldsymbol{w}_g}  \mathcal{L}_{\rho}(\boldsymbol{w}_1, \dots, \boldsymbol{w}_G, \boldsymbol{w}, \boldsymbol{\mu}_1, \dots, \boldsymbol{\mu}_G)\\
    &\label{eq:admm_2a} =\inf_{\boldsymbol{w}_g} \mathcal{L}_{\rho_g}(\boldsymbol{w}_g, \boldsymbol{w}, \boldsymbol{\mu}_g)\\
    &\label{eq:admm_2b} = \inf_{\boldsymbol{w}_g}\sup_{\mathbb{Q}_g \in \mathcal{A}^{(g)}_{\varepsilon_g,1,d}(\Xi)}\mathbb{E}^{\mathbb{Q}_g}[\ell_H(\boldsymbol{w}_g;\boldsymbol{\xi})]+ \frac{\rho}{2} ||\boldsymbol{w}_g - \boldsymbol{w} + \boldsymbol{\mu}_g||_2^2\\
    &\label{eq:admm_2c} \begin{aligned}&= \inf_{\boldsymbol{w}_g,\lambda_g \geq 0} \lambda_g \varepsilon_g + \frac{1}{N_g} \sum_{n_g=1}^{N_g} \sup_{\boldsymbol{\xi} \in \Xi} \left \{\ell_H(\boldsymbol{w}_g;\boldsymbol{\xi}) - \lambda_g d(\boldsymbol{\xi},\widehat{\boldsymbol{\xi}}_{n_g}) \right \} + \frac{\rho}{2} ||\boldsymbol{w}_g - \boldsymbol{w} + \boldsymbol{\mu}_g||_2^2 \end{aligned}\\
    &\label{eq:admm_2d_} = \left \{ \begin{aligned}
        & \min_{\boldsymbol{w}_g,\lambda_g,s_{n_g}}&& \lambda_g \varepsilon_g + \frac{1}{N_g} \sum_{n_g=1}^{N_g}s_{n_g} +\frac{\rho}{2} || \boldsymbol{w}_g - \boldsymbol{w} + \boldsymbol{\mu}_g||_2^2&&\\
        & \subjectto && \ell_H(\boldsymbol{w}_g;(\widehat{\boldsymbol{x}}_{n_g},\widehat{y}_{n_g})) \leq s_{n_g}&& \forall n_g \in [N_g] \\
        &&&\ell_H(\boldsymbol{w}_g;(\widehat{\boldsymbol{x}}_{n_g},-\widehat{y}_{n_g})) - \kappa \lambda_g \leq s_{n_g} &&\forall n_g \in [N_g] \\
        &&& \lambda \geq ||\boldsymbol{w}_g||_{\ast}&&
    \end{aligned} \right.,
\end{align}
\end{subequations}
where \ref{eq:admm_2a} follows from the separability of the Augmented Lagrangian, \ref{eq:admm_2b} follows by definition of the local Lagrangian, \ref{eq:admm_2c} exploits the notable duality result presented by \cite{datadrivenDRO,kuhn2019wasserstein} to rewrite the inner maximization problem as a minimization problem, and
\eqref{eq:admm_2d_} follows by introducing slack variables $s_{n_g}$, recalling that $\ell_H(\boldsymbol{w};\boldsymbol{\xi})$ is convex and Lipschitz continuous, and utilizing similar arguments to the ones presented in the proof of Theorem 1 in \citep{shafieezadehabadeh2015distributionally}.

\subsubsection{Proof of Proposition \ref{prop:admm_server_up}} \label{proof:admm_server_up}
\begin{proof}
The central server can obtain updated global parameters $\boldsymbol{w}^{\ast}$ by minimizing the global Lagrangian with respect to $\boldsymbol{w}$. This can be done as follows.
\begin{subequations}
\begin{align}
    \boldsymbol{w}^{\ast} &= \argmin_{\boldsymbol{w}} \mathcal{L}_{\rho}(\boldsymbol{w}_1, \dots, \boldsymbol{w}_G, \boldsymbol{w}, \boldsymbol{\mu}_1, \dots, \boldsymbol{\mu}_G)\\
    &\label{eq:admm_3a}= \argmin_{\boldsymbol{w}} \sum_{g=1}^G \alpha_g\mathcal{L}_{\rho_g}(\boldsymbol{w}_g, \boldsymbol{w}, \boldsymbol{\mu}_g) \\
    &\label{eq:admm_3b} = \argmin_{\boldsymbol{w}} \sum_{g=1}^G \alpha_g \frac{\rho}{2}||\boldsymbol{w}_g - \boldsymbol{w} + \boldsymbol{\mu}_g||_2^2,
\end{align}
\end{subequations}
where \ref{eq:admm_3b} follows by observing that the norm term is the only term involving the variable $\boldsymbol{w}$. Let us define $f(\boldsymbol{w}) = \sum_{g=1}^G \alpha_g \frac{\rho}{2} ||\boldsymbol{w} - (\boldsymbol{w}_g + \boldsymbol{\mu}_g)||_2^2$. We note that $f(\boldsymbol{w})$ is strongly convex as it is a sum of strongly convex terms. Thus, it has a unique minimizer. We analyze its partial derivative with respect to $\boldsymbol{w}$ by setting it to 0 to obtain our minimizer as follows.
\begin{subequations}
\begin{align}
    \frac{\partial f}{\partial \boldsymbol{w}} &= \sum_{g=1}^G \alpha_g\frac{\rho}{2} \left [2 \boldsymbol{w} - 2(\boldsymbol{w}_g + \boldsymbol{\mu}_g) \right ]\\
    & = 0
\end{align}
\end{subequations}
Finally, we derive a closed form solution for $\boldsymbol{w}^{\ast}$ as follows:
\begin{subequations}
\begin{align}
    &\sum_{g=1}^G \alpha_g\frac{\rho}{2} \left [2 \boldsymbol{w} - 2(\boldsymbol{w}_g + \boldsymbol{\mu}_g) \right ] = 0\\
     \Leftrightarrow &\sum_{g=1}^G \alpha_g \boldsymbol{w} = \sum_{g=1}^G \alpha_g (\boldsymbol{w}_g + \boldsymbol{\mu}_g)\\
     \Leftrightarrow&\label{eq:admm_4a} \boldsymbol{w} = \sum_{g=1}^G \alpha_g (\boldsymbol{w}_g + \boldsymbol{\mu}_g),
\end{align}
\end{subequations}
where \ref{eq:admm_4a} follows by recalling that $\sum_{g=1}^G \alpha_g = 1$.
\end{proof}

\subsubsection{Proof of Theorem \ref{thm:admm_conv_cri}} \label{proof:admm_conv_cri}
\begin{proof}
    As mentioned previously, even when the client objective functions are closed proper convex functions as we demonstrate in \ref{lem:clo_pro_con}, and strong duality holds as shown by \cite{kuhn2019wasserstein}, multi-block ADMM is not theoretically guaranteed to converge \citep{chen2016}. However, \cite{lin2015} establish the convergence of multi-block ADMM in the setting where the objective functions of $(B-1)$ of the $B$ blocks are strongly convex with strong convexity parameter $\sigma_b$ for each block $b$. They formulate the problem to be solved via ADMM as follows:
    \begin{equation} \label{eq:admm_paper_form}
        \begin{aligned}
            &\min && f_1(\boldsymbol{v}_1) + f_2(\boldsymbol{v}_2) + \dots + f_B(\boldsymbol{v}_B)\\
            & \subjectto && \boldsymbol{A}_1\boldsymbol{v}_1 + \boldsymbol{A}_2\boldsymbol{v}_2 + \dots + \boldsymbol{A}_B\boldsymbol{v}_B = c\\
            &&& \boldsymbol{v}_b \in \mathcal{V}_b \quad \forall b \in [B],
        \end{aligned}
    \end{equation}
    where $f_b(\boldsymbol{v}_b)$ is the objective function term and $\boldsymbol{v}_b$ is the decision variable associated with the $b^{th}$ block. 
    
    Note that if we were to rewrite our problem from \eqref{eq:admm_orig_obj} in a similar form, there would be no distinction between the clients and the central server, and the objective function term associated with the central server would remain $0$. Thus, we add a strongly convex term $\tau_g||\boldsymbol{w}_g||_2^2$ to the objective function term associated with each of the clients to meet the requirement that $B-1$ of the blocks must have a strongly convex objective function. During the server aggregation step, each $\tau_g||\boldsymbol{w}_g||_2^2$ term will be multiplied by its respective weight $\alpha_g$. Therefore, the strong convexity parameter associated with client $g$ would be $2\alpha_g\tau_g$.

   To rewrite problem \eqref{eq:admm_orig_obj} in the form of problem \eqref{eq:admm_paper_form}, the $\boldsymbol{A}$ matrix associated with client $g$ would be a block matrix of $P\times P$ matrices stacked vertically in $G$ blocks. The $g^{th}$ block from the top would be the identity matrix, whereas all the other blocks would be zero. Similarly, the matrix associated with the central server would be a block matrix of similar structure but where all the blocks are the negative of the identity matrix. Incorporating this insight into the condition on $\rho$ described in Theorem 3.3 in \citep{lin2015} allows us to obtain the final result.
\end{proof}

\subsubsection{Proof of Theorem \ref{thm:admm_conv}} \label{proof:admm_conv}
\begin{proof}
This proof follows a very similar strategy to that of Theorem \ref{thm:subgrad_conv}. We begin by noting that the strongly convex variant of the local model problem in \eqref{eq:admm_2d} equipped with the $\ell_1$-norm can be written as a quadratically constrained quadratic problem (QCQP) with $N_g+2P+3$ decision variables (including slack variables) and $2N_g+2P+3$ constraints. When solved via the barrier method equipped with the log barrier function and Newton updates \citep{nn1994}, this problem would have the following worst-case time complexity
    \begin{equation*}
        \mathcal{O}([N_g+P]^{3.5}\log(\epsilon_2^{-1})).
    \end{equation*}
    Similar to the previous algorithm, all clients can solve their local problems in parallel and will have the same number of features. Thus the client with the greatest number of samples $N_{g^{\ast}}$ will have the problem with the greatest time complexity. Furthermore, we note that the central server aggregates $2G$ vectors of dimension $P$ in each iteration, the time complexity of which is $\mathcal{O}(GP)$. Therefore, we obtain the final result by noting that ADMM converges to an $\epsilon_1$-solution in $\mathcal{O}(\epsilon_1^{-1})$ iterations assuming the strong convexity of the objective function and that the upper bound on $\rho$ is satisfied \citep{lin2015}. While each client $g$ also performs the update of the local scaled Lagrange multipliers $\boldsymbol{\mu}_g$ during each iteration, this process has a much lower complexity than solving the local problem and is, therefore, not explicitly considered in this analysis.
\end{proof}

\subsection{Supplementary Theoretical Results}
\subsubsection{Strongly Convex ADMM Client Update Problem}\label{sec:admm_sc}
In the main body of the paper we presented the optimization problem to be solved locally by each client during each round of our proposed ADMM algorithm \ref{alg:admm}. Below, we present the strongly convex version of this problem $J_g^{\text{SC}}(\boldsymbol{w},\boldsymbol{\mu}_g)$, which theoretically guarantees the convergence of the algorithm. This is the version utilized by the ADMM-SC algorithm. Please note that the proof for this formulation is exactly the same as that of Proposition \ref{prop:admm_client_up}.
\begin{equation*}
    J_g^{\text{SC}}(\boldsymbol{w},\boldsymbol{\mu}_g) \coloneqq \left \{ \begin{aligned}
        & \min_{\boldsymbol{w}_g,\lambda_g,s_{n_g}}&& \lambda_g \varepsilon_g + \frac{1}{N_g} \sum_{n_g=1}^{N_g}s_{n_g}+\frac{\rho}{2} || \boldsymbol{w}_g - \boldsymbol{w} + \boldsymbol{\mu}_g||_2^2 + \tau_g||\boldsymbol{w}_g||_2^2&&\\
        & \subjectto && \ell_H(\boldsymbol{w}_g;(\widehat{\boldsymbol{x}}_{n_g},\widehat{y}_{n_g})) \leq s_{n_g} && \forall n_g \in [N_g] \\
        &&&\ell_H(\boldsymbol{w}_g;(\widehat{\boldsymbol{x}}_{n_g},-\widehat{y}_{n_g})) - \kappa \lambda_g \leq s_{n_g} && \forall n_g \in [N_g] \\
        &&& \lambda \geq ||\boldsymbol{w}_g||_{\ast},&&
    \end{aligned} \right.
\end{equation*}
 where $||\cdot||_{\ast}$ is the dual of the norm used in \ref{eq:cost_func}.
\section{Further Experimental Details and Supplementary Results}\label{app:exp}
In this section we provide all the details of all the experiments presented in this paper, as well as the results of a \textbf{scalability experiment}. Please note that the all the code and instructions associated with all the experiments are available at \href{https://github.com/mibrahim41/FDR-SVM}{this link}.
\subsection{Software and Hardware Details}
All the experiments presented in this work were executed on Intel Xeon Gold 6226 CPUs @ 2.7 GHz (using 4 cores) with 120 Gb of DDR4-2993 MHz DRAM. Table \ref{tab:software} provides more detail on all the software used in the paper.

\begin{table}[ht]
\centering
\caption{Details on All the Software Used in the Numerical Experiments.}
\label{tab:software}
\begin{tabular}{lll}
\toprule
Software           & Version & License             \\
\midrule
Gurobi       & 10.0.1  & Academic license    \\
MATLAB       & R2021B  & Academic license    \\
Python     & 3.10.9  & Open source license \\
Scikit-Learn & 1.2.1   & Open source license \\
Numpy      & 1.23.5  & Open source license \\
Scipy       & 1.10.0  & Open source license \\
UCIMLRepo     & 0.0.3   & Open source license\\
\bottomrule
\end{tabular}
\end{table}

\subsection{Datasets Utilized}
\subsubsection{UCI Data Experiment}
We provide details on the datasets used in the experiment described in Section \ref{sec:real}. Note that Parkinson's exhibited very high levels of class imbalance ($75\%$ from one class and $25\%$ from the other), which suggests that the SM algorithm is more successful with data that exhibits such levels of imbalance. Moreover, note that the "Very Low" and "Low" classes in the UKM dataset were combined into one class, whereas "Middle" and "High" were combined into another.

\begin{table}[ht]
\centering
\caption{Details on Datasets Utilized for UCI Experiments.}
\label{tab:data_real}
\begin{tabular}{lll}
\toprule
Dataset                              & Abbreviation & License   \\
\midrule
Banknote Authentication \citep{banknote}             & Banknote     & CC BY 4.0 \\
Breast Cancer Wisconsin (Diagnostic) \citep{bcw}& BCW          & CC BY 4.0 \\
Connectionist Bench (Sonar) \citep{cb}         & CB           & CC BY 4.0 \\
Mammographic Mass    \citep{mm}                & MM           & CC BY 4.0 \\
Parkinson's     \citep{parkinsons}                     & Parkinson's  & CC BY 4.0 \\
Rice (Cammeo and Osmancik)   \citep{rice}        & Rice         & CC BY 4.0 \\
User Knowledge Modeling   \citep{ukm}           & UKM          & CC BY 4.0\\
\bottomrule
\end{tabular}
\end{table}

\subsubsection{Industrial Data Experiment}
The data used in the experiment described in Section \ref{sec:sens} is a simulation dataset that uses a physics-driven Simulink model to simulate the healthy and faulty operation of a reciprocating pump \citep{Mathworks}. The generated simulation data belongs to two classes: healthy pump and leak fault. We focus on the binary classification problem since binary classification models can directly extend to multiclass problems via a one-vs-all framework as mentioned previously. Therefore, performance in the binary setting is indicative of that in the multiclass setting. However, data is generated to simulate different severities of the leak fault, where each client has a different severity to simulate data heterogeneity across clients. Note that leak fault severity is controlled via a \texttt{leak\_area\_set\_factor} variable in the MATLAB script. The four values used in our experiments are $[1e-3,4e-3,7e-3,1e-2]$. Features extracted from the generated time series data (such as kurtosis and skewness) are used for classification. 

\subsection{Hyperparameter Details}
In all of our implementations of the SM algorithm we utilize a step-size that diminishes according to $\gamma(t) = \frac{\gamma}{t}$, where we treat $\gamma$ as a model hyperparameter. This step-size obeys the conditions required for algorithm convergence stated in Theorem \ref{thm:sm_conv}. Next, we provide details on the hyperparameter values used in the UCI Data Experiment and the Industrial Data Experiment in Sections \ref{sec:real} and \ref{sec:sens}, respectively.

\textbf{UCI Data Experiment.} For the centralized baseline, we tune $\varepsilon \in \{ 1 \times 10^b \}_{b=-5}^{-1}$ and $\kappa \in \{ 0.1,0.25,0.5,0.75,1\}$. For the federated baselines we use diminishing step-size of $\gamma(t) = \frac{\gamma(0)}{t}$, where $\gamma(0)$ is treated as a tuning hyperparameter and takes values $\gamma(t) \in \{1e-3,1e-2,1e-1,1e0\}$, and a local regularization penalty of $\frac{1}{10N_g}$ at each client. For \texttt{FedAvg} and \texttt{FedProx}, we utilize a local batch size of $20\%$ of the available training data, and $E=5$ local SGD epochs where appropriate. We also use a $\mu = 1$ for \texttt{FedProx}. For our proposed methods we fix $\kappa_g = 1$ and $\varepsilon_g = \frac{1}{10N_g}$, and we tune $\rho \in \{1e-3,1e-2,1e-1,1e0 \}$ and $\gamma \in \{1e0,1e1,1e2,1e3 \}$. Finally, for all federated methods (including baselines and ours) we use $G=4$ with equal client weights. Finally, for ADMM, ADMM-SC and federated baselines, we tune $T \in \{5,10,20,60,100,140,180,220\}$, whereas for SM we tune $T \in \{100,140,180,220\}$. All tuning is done via 5-fold cross-validation.

\textbf{Industrial Data Experiment - Sensitivity Analysis.} In the global hyperparameters experiment we evaluate the performance of our proposed federated algorithms for $T \in \{5,10,20,60,100,140,180,220\}$ and $\rho,\gamma \in \{1e-3,1e-2,1e-1,1e0,1e1,1e2,1e3 \}$. We fix $\varepsilon_g = \frac{1}{10N_g}$ and $\kappa_g = 0.5$ for each client $g$. While such values of $\varepsilon_g$ and $\kappa_g$ may not be optimal, we use them to demonstrate that our proposed model can perform well when compared to the central baseline.

In the local hyperparameters testing, We evaluate the performance of both the our federated algorithms for $\varepsilon_g = \frac{1}{\beta N_g}$ where $\beta \in \{ 0.1,1,10,100 \}$ and for $\kappa_g = \kappa \in \{ 0.1,0.25,0.5,0.75,1 \}$. We fix $T = 220$ and $\gamma = 1\times 10^2$ and $T = 100$ and $\rho = 1\times 10^{-3}$ for the SM and ADMM algorithms, respectively. In all settings we evaluate the performance of the baseline central model for $\varepsilon \in \{ 1 \times 10^b \}_{b=-5}^{-1}$ and $\kappa \in \{ 0.1,0.25,0.5,0.75,1\}$, and we only report the peak performance achieved. 

In all settings we utilize $\tau_g=18\rho$ for the ADMM-SC algorithm, which is the minimum value $\tau_g$ can take while maintaining guaranteed convergence as shown in Theorem \ref{thm:admm_conv_cri}. We do this as increasing $\tau$ increases the strength of the redundant regularization, thereby impacting the performance.

\textbf{Industrial Data Experiment - Benchmarking.} In this portion we utilize 5-fold cross-validation to tune the same hyperparameters discussed in the previous paragraph. Namely, we fix $T=220$, and we tune $\rho \in \{ 1e-3,1e-2,1e-1 \}$ or $\gamma \in \{ 1e1, 1e2, 1e3 \}$, $\kappa \in \{ 0.1, 0.5, 1 \}$, and $\beta \in \{ 10, 100\}$ for all our methods. Tuning is done via 5-fold cross-validation.

\textbf{Model Parameter Initialization.} In all of our experiments, we use initial model parameters $\boldsymbol{w}^{(0)} = \boldsymbol{0}$ (i.e., a vector of zeros), and initial scaled Lagrange multipliers $\boldsymbol{\mu}_g^{(0)} = \boldsymbol{1}$ (i.e., a vector of ones).

\subsection{UCI Data Experiment Statistical Significance} \label{app:stat_sig}
In order to evaluate the statistical significance of the results presented in Table \ref{tab:real_world}, we perform a one-sided Wilcoxon signed-rank test. The test compares the performance of the best performing version of our model to that attained by each of the benchmarks in a pairwise fashion. The null $H_0$ and alternative $H_1$ hypotheses of this test are defined next.
\begin{itemize}
    \item $H_0$: The distribution of the differences in performance between our model and each benchmark has median zero. That is, there is no systematic increase or decrease between the pairs.
    \item $H_1$:  The median of the differences is greater than $0$. That is, our approach is statistically better.
\end{itemize}

The results of this test are presented in Table \ref{tab:wilc}, utilizing a significance level of $\alpha = 0.05$. The table indicates whether the null hypothesis $H_0$ is rejected or not. We observe from the table that the performance improvement offered by our model algorithm is indeed statistically significant for most datasets and most benchmark models. This is because we "Reject" the null hypothesis $H_0$ in most settings. This underscores the practical impact and performance improvements offered by our proposed model.

\begin{table*}[ht]
\centering
\caption{Results of One-Sided Wilcoxon Signed-Rank Test Performed on Results of Benchmarking Experiments on 7 UCI Datasets.}
\label{tab:wilc}
\smaller
\begin{tabular}{lccccccc}
\toprule
Model    & \multicolumn{1}{c}{Banknote} & \multicolumn{1}{c}{BCW} & \multicolumn{1}{c}{CB} & \multicolumn{1}{c}{MM} & \multicolumn{1}{c}{Parkinson's} & \multicolumn{1}{c}{Rice} & \multicolumn{1}{c}{UKM} \\ \midrule
FedSGD ($\ell_2$-SVM)  &Fail to reject&Reject&Reject&Reject&Reject&Reject& Reject\\
FedAvg ($\ell_2$-SVM) &Fail to reject&Reject&Fail to reject&Reject&Reject&Fail to reject& Fail to reject \\
FedProx ($\ell_2$-SVM) &fail to reject&Reject&Fail to reject&Reject&Reject&Reject& Fail to reject \\ 
FedDRO (KL) &Reject& Reject & Reject &Reject&Reject& Reject&Reject \\
\bottomrule
\end{tabular}
\end{table*}
\subsection{Scalability Experiment}\label{sec:scal}
The purpose of this experiment is to examine the scalability of our proposed algorithms as the total number of samples $N$, the number of clients $G$, and the number of features $P$ grow. We measure performance in runtime required to achieve peak mCCR. We do this since the subgradient method lacks a practically implementable stopping criterion \citep{Bagirov2014}, and similarly, no stopping criterion is provided for multi-block ADMM by \cite{lin2015}. Moreover, it was already established in the experiment in Section \ref{sec:sens} that the SM algorithm requires more rounds of communication to attain peak performance. This experiment is more focused on computational effort required to achieve this performance. We examine the the following settings: 
\begin{enumerate}
    \item \textit{Increasing clients [fixed training samples]}: $N = 1000$, $P = 4$, $G \in \{10,20,30,40,50\}$.
    \item \textit{Increasing clients [increasing training samples]}: $N = 100G$, $P = 4$, $G \in \{10,20,30,40,50\}$.
    \item \textit{Increasing training samples}: $G = 10$, $P = 4$, $N \in \{1000,1500,2000,2500,3000\}$.
    \item \textit{Increasing features}: $N = 4$, $G = 10$, $P \in \{4,6,8,10,12\}$.
\end{enumerate}

\textbf{Dataset.} This experiment uses simulation data that is generated using the \texttt{make\_classification} module of the Scikit-Learn Python package \citep{scikit-learn}. The data generated belongs to two classes, each of which contains data sampled from a standard Gaussian distribution with means located at vertices of a $P$-dimensional hypercube with sides of length $2.4$ centered at the origin. The data is distributed equally across all clients and both classes, and no labels are altered.

\textbf{Baseline.} We utilize the centralized DR-SVM by \cite{2019regularization} as a baseline in this experiment.

\textbf{Hyperparameters.} For the SM algorithm, we test performance for $T \in \{ 140,180,220 \}$ and $\gamma \in \{ 1e1,1e2,1e3 \}$. For the ADMM and ADMM-SC algorithms, we test performance for $T \in \{ 10,20,30 \}$ and $\rho \in \{ 1e-3,1e-2,1e-1\}$. Across all algorithms, we fix $\varepsilon_g = \frac{1}{10N_g}$ and $\kappa = 0.25$. The central model's hyperparameters are varied in the same way as in the Sensitivity Analysis portoin of the experiment in Section \ref{sec:sens}, and the runtime that is reported reflects the time taken to solve the optimization problem. 

\textbf{Results.} The results of this study are reported in Figure \ref{fig:scal}. We observe a rough trend of increasing runtime as $N$ and $P$ increase for all models due to the increasing complexity of the local client problems. However, the trend is clearer with the SM algorithm, whereas it is noisy with all versions of the ADMM algorithm, and is hardly observable with the central model. This could be attributed to the fact that the SM algorithm requires a much longer time to reach peak mCCR, making the effect of random computer system variations minimal on the reported time. On the contrary, all versions of the ADMM and the central model reach peak mCCR in a very short time, making the reported time highly susceptible to system variations. These results highlight the fact that any performance gains achieved by using SM come at the cost of a much longer runtime. However, the runtime of ADMM and ADMM-SC is much closer to that of the central approach. Additionally, we observe that the runtime remains roughly constant for all federated algorithms as $G$ increases if $N$ is fixed. This is because a fixed $N$ makes the local problem at each client increasingly simpler and faster to solve as $G$ increases. In contrast, when both $G$ and $N$ are increasing we observe that all algorithms exhibit a trend of increasing runtime with the number of clients. 

\begin{figure*}[ht]
    \centering
    \includegraphics[width=1\textwidth]{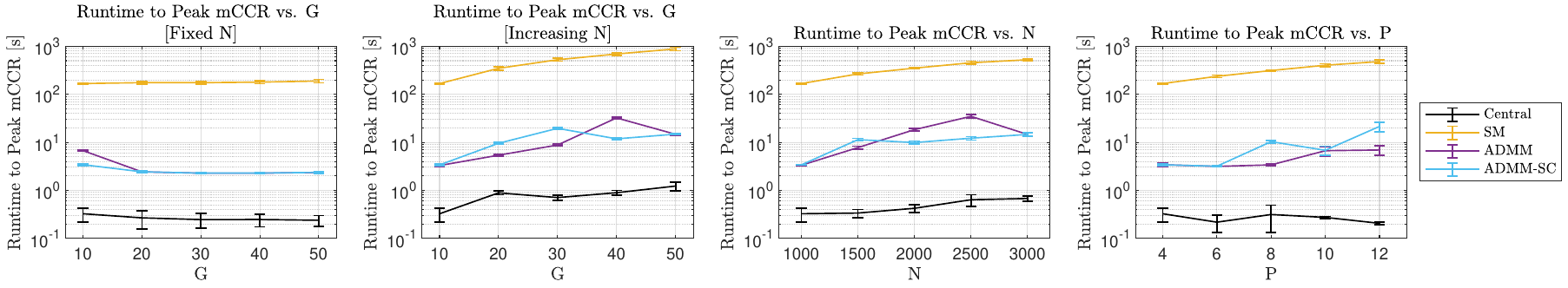}
    \caption{Plots of Runtime to Reach Peak mCCR vs. the Number of Clients $G$ with Fixed and Increasing $N$, the Number of Features $P$, and the Number of Training Samples $N$ for All Methods Tested.}
    \label{fig:scal}
\end{figure*}

\end{document}